\documentclass[11pt]{amsart}

\usepackage{amsmath,amssymb,mathtools}

\usepackage{amssymb}
\usepackage{amsfonts}
\usepackage{color}
\usepackage[linktoc=page, colorlinks, linkcolor=blue, citecolor=blue]{hyperref}
\usepackage{bookmark}
\usepackage{xcolor}

\usepackage{float}

\usepackage{algorithmic}
\usepackage{algorithm, algorithmic}

\newtheorem{definition}{Definition}[section]

\newtheorem{theorem}[definition]{Theorem}
\newtheorem{lemma}[definition]{Lemma}

\newtheorem{remark}[definition]{Remark}

\newcommand{\xkh}[1]{\left(#1\right)}
\newcommand{\dkh}[1]{\left\{#1\right\}}

\newcommand{\R}{\mathbb{R}}

\newcommand{\N}{\mathbb{N}}

\newcommand{\EEb}{\boldsymbol{E}}

\newcommand{\overleq}[1]{\overset{#1}{\le}}
\newcommand{\overeq}[1]{\overset{#1}{=}}
\newcommand{\bracing}[2]{\underset{#1}{\underbrace{#2}}}

\newcommand{\twonorm}[1]{\big\Vert #1 \big\Vert_2}
\newcommand{\specnorm}[1]{\big\Vert #1 \big\Vert_2}
\newcommand{\normf}[1]{\big\Vert #1 \big\Vert_\mathrm{F}}

\newcommand{\nj}[1]{\langle #1 \rangle}

\makeatletter
\def\@tvsp{\mathchoice{{}\mkern-4.5mu}{{}\mkern-4.5mu}{{}\mkern-2.5mu}{}}
\def\ltrivert{\left|\@tvsp\left|\@tvsp\left|}
\def\rtrivert{\right|\@tvsp\right|\@tvsp\right|}
\makeatother

\newcommand{\bm}{\boldsymbol}

\newcommand{\AAi}{\boldsymbol{A}_i}

\newcommand{\AAiw}{\boldsymbol{A}_{i}^{\left(\ww,\vv\right)}}
\newcommand{\AAf}{\boldsymbol{A}}
\newcommand{\ZZ}{\boldsymbol{Z}}
\newcommand{\BB}{\boldsymbol{B}}
\newcommand{\UU}{\boldsymbol{U}}
\newcommand{\RR}{\boldsymbol{R}}
\newcommand{\OO}{\boldsymbol{O}}
\newcommand{\VV}{\boldsymbol{V}}
\newcommand{\WW}{\boldsymbol{W}}
\newcommand{\GG}{\boldsymbol{G}}
\newcommand{\LL}{\boldsymbol{L}}
\newcommand{\QQ}{\boldsymbol{Q}}

\newcommand{\XXt}{\boldsymbol{X}_t}
\newcommand{\XXtw}{\boldsymbol{X}_{t}^{\left(\ww,\vv \right)}}
\newcommand{\XXtplus}{\boldsymbol{X}_{t+1}}
\newcommand{\XXtwplus}{\boldsymbol{X}_{t+1}^{\left(\ww,\vv \right)}}

\newcommand{\LLt}{\boldsymbol{L}_t}
\newcommand{\LLtw}{\boldsymbol{L}_{t}^{\left(\ww,\vv \right)} }
\newcommand{\LLtplusw}{\boldsymbol{L}_{t+1}^{\left(\ww,\vv \right)} }
\newcommand{\LLtT}{\boldsymbol{L}_t^\top}
\newcommand{\LLtplus}{\boldsymbol{L}_{t+1}}
\newcommand{\LLtwT}{\boldsymbol{{L}_{t}^{\left(\ww,\vv \right)}}^\top  }

\newcommand{\VLLt}{\boldsymbol{V}_{t}}

\newcommand{\VLLtP}{\boldsymbol{V}_{t}^\top}

\newcommand{\VLLtw}{\boldsymbol{V}_{t}^{\left(\ww,\vv \right)}}

\newcommand{\VLLtwP}{\boldsymbol{{V}_{t}^{\left(\ww,\vv \right)}}^\top}

\newcommand{\RRt}{\boldsymbol{R}_t}
\newcommand{\RRtw}{\boldsymbol{R}_{t}^{\left(\ww,\vv \right)} }
\newcommand{\RRtwT}{\boldsymbol{{R}_{t}^{\left(\ww,\vv \right)}}^\top }
\newcommand{\RRtplusw}{\boldsymbol{R}_{t+1}^{\left(\ww,\vv \right)} }
\newcommand{\RRtT}{\boldsymbol{R}_t^\top}
\newcommand{\RRtplus}{\boldsymbol{R}_{t+1}}

\newcommand{\VRRtP}{\boldsymbol{W}_{t}^\top}
\newcommand{\VRRt}{\boldsymbol{W}_{t}}

\newcommand{\VRRtw}{\boldsymbol{W}_{t}^{\left(\ww,\vv \right)}}

\newcommand{\VRRtwP}{\boldsymbol{{W}_{t}^{\left(\ww,\vv \right)}}^\top}

\newcommand{\WXX}{\boldsymbol{W}_{\star}}
\newcommand{\WXXT}{\boldsymbol{W}_{\star}^\top}
\newcommand{\WXXP}{\boldsymbol{W}_{\star,\bot}}
\newcommand{\WXXPT}{\boldsymbol{W}_{\star,\bot}^\top}

\newcommand{\XXstar}{\boldsymbol{X}_{\star}}

\newcommand{\XX}{\boldsymbol{X}}
\newcommand{\VXX}{\boldsymbol{V}_{\star}}

\newcommand{\VXXT}{\boldsymbol{V}_{\star}^\top}
\newcommand{\VXXP}{\boldsymbol{V}_{\star,\bot}}
\newcommand{\VXXPT}{\boldsymbol{V}_{\star,\bot}^\top}
\newcommand{\MM}{\boldsymbol{M}}

\newcommand{\NN}{\boldsymbol{N}}
\newcommand{\SSigma}{\boldsymbol{\Sigma}}

\newcommand{\II}{\boldsymbol{I}}
\newcommand{\Sg}{\boldsymbol{\Sigma}}

\newcommand{\bS}{\mathbb S}

\newcommand{\IdOp}{\mathcal{I}}
\newcommand{\vv}{\boldsymbol{v}}

\newcommand{\yy}{\boldsymbol{y}}
\newcommand{\xx}{\boldsymbol{x}}
\newcommand{\ww}{\boldsymbol{w}}

\newcommand{\LLambda}{\boldsymbol{\Lambda}}

\DeclareMathOperator{\dist}{dist}
\DeclareMathOperator{\st}{\mathbf{s.t.}}
\DeclareMathOperator{\rank}{rank}
\DeclareMathOperator{\aand}{and}

\newcommand{\Aop}{\mathcal{A}}
\newcommand{\Aopw}{\mathcal{A}_{\left(\ww,\vv\right)}}
\newcommand{\Aops}{\Aop^* \Aop}
\newcommand{\Aopws}{\Aopw^* \Aopw}

\newcommand{\Projw}{\mathcal{P}_{\ww \vv^\top}}
\newcommand{\Projwperp}{\mathcal{P}_{\ww \vv^\top,\bot}}

\newcommand{\supw}{\underset{\left(\ww,\vv \right) \in \mathcal{N}}{\sup}}

\newcommand{\constone}{c_1}
\newcommand{\consttwo}{c_2}
\newcommand{\constthree}{c_3}

\newcommand{\constfive}{c_5}

\title[Scaled Gradient Descent for Ill-Conditioned Low-Rank Matrix Recovery]{Scaled Gradient Descent for Ill-Conditioned Low-Rank Matrix Recovery with Optimal Sampling Complexity}

\author{Zhenxuan Li} 
\address{School of Mathematical Sciences, Beihang University, Beijing, 100191, China}
\email{Lizhenxuan@buaa.edu.cn }

\author{Meng Huang}
\address{School of Mathematical Sciences, Beihang University, Beijing, 100191, China}
\email{menghuang@buaa.edu.cn}
\thanks{M. Huang was supported by Beijing Natural Science Foundation (1262013) and the National Nature Science Foundation of China (12201022).}

\begin{document}

\maketitle

\begin{abstract}
The low-rank matrix recovery problem seeks to reconstruct an unknown $n_1 \times n_2$ rank-$r$ matrix from $m$ linear measurements, where $m\ll n_1n_2$. This problem has been extensively studied over the past few decades, leading to a variety of algorithms with solid theoretical guarantees. Among these, gradient descent based non-convex methods have become particularly popular due to their  computational efficiency. However, these methods typically suffer from two key limitations: a sub-optimal sample complexity of $O((n_1 + n_2)r^2)$ and an iteration complexity of $O(\kappa \log(1/\epsilon))$ to achieve $\epsilon$-accuracy, resulting in slow convergence when the target matrix is ill-conditioned. Here, $\kappa$ denotes the condition number of the unknown matrix. Recent studies show that a preconditioned variant of GD, known as scaled gradient descent (ScaledGD), can significantly reduce  the iteration complexity to $O(\log(1/\epsilon))$. Nonetheless, its sample complexity remains sub-optimal at $O((n_1 + n_2)r^2)$. In contrast, a delicate virtual sequence technique demonstrates that the standard GD  in the positive semidefinite (PSD) setting achieves the optimal sample complexity $O((n_1 + n_2)r)$, but converges more slowly with an iteration complexity $O(\kappa^2 \log(1/\epsilon))$.
In this paper, through a more refined analysis, we show that ScaledGD achieves  both the optimal sample complexity $O((n_1 + n_2)r)$ and  the improved iteration complexity $O(\log(1/\epsilon))$. Notably, our results extend beyond the PSD setting to general low-rank matrix recovery problem. Numerical experiments further validate that ScaledGD accelerates  convergence for ill-conditioned matrices with the optimal sampling complexity.
\end{abstract}


\section{Introduction}\label{sec:introduction}
\subsection{Problem setup}
Low-rank matrix recovery has wide-ranging applications in machine learning \cite{CandesRecht_2012}, recommendation systems \cite{10.1007/s10115-021-01629-6}, imaging science \cite{10.1109/TIP.2015.2432712}, and other areas \cite{davenport2016overview,chen2018harnessing}. It encompasses several classical problems, including matrix completion \cite{CandesRecht_2012}, phase retrieval \cite{Cand2013PhaseLift}, robust PCA \cite{10.1145/1970392.1970395}, blind deconvolution \cite{10.1109/TIT.2013.2294644}, and blind demixing \cite{10.1109/TIT.2017.2701342}. Broadly speaking, these problems can often be cast as solving the following non-convex program:
\begin{align}  \label{eq:prime problem}
      \min_{\XX \in \mathbb{\R}^{n_1 \times n_2}}& f(\XX):=\frac{1}{2}\|\bm{y}-\Aop\left(\XX\right)\|_2^2  \notag \\
     &\st \quad \rank(\XX)\le r,
\end{align}
where $\yy:=\Aop(\XXstar)\in \mathbb{R}^m$ is the observed measurement vector. Here, $\XXstar \in \mathbb{R}^{n_1 \times n_2}$ is the rank-$r$ matrix to be recovered, and $\Aop: \mathbb{R}^{n_1 \times n_2}\rightarrow \mathbb{R}^m$ is  linear operator  defined by
\begin{equation*}
 \left[\Aop (\XX) \right]_i :=  \frac1{\sqrt m}\nj{\AAi, \XX},
    \quad \quad i=1,2,\ldots, m,
\end{equation*}
where $\AAi \in \mathbb{R}^{n_1 \times n_2}$ are known measurement matrices, $\nj{\AAi, \XX}:=\operatorname{trace}(\AAi^\top\XX)$ denotes the standard inner product, and $m\ll n_1n_2$.

A commonly used and efficient strategy for solving \eqref{eq:prime problem}  is to parametrize the low-rank matrix as $\XX=\LL\RR^\top$, where $\LL \in \mathbb{\R}^{n_1\times r}$ and $\RR \in \mathbb{\R}^{n_2\times r}$, so that \eqref{eq:prime problem} is  reformulated as \cite{Burer2003nonlinear,Boumal2016nonconvex,li2016rapidrobustreliableblind,keshavan2009matrixcompletionentries}
\begin{equation}   \label{eq:parametrize problem}
    \min_{\LL\in \mathbb{\R}^{n_1\times r},\RR \in \mathbb{\R}^{n_2\times r}} \mathcal{L} {\left(\LL, \RR\right)}:=\frac{1}{2}\|\yy-\Aop\left(\LL \RR^\top\right)\|_2^2.
\end{equation}
Although \eqref{eq:parametrize problem} is non-convex, under the Gaussian design where each $\AAi$ is a standard Gaussian random matrix,  it has been shown  that simple gradient descent with spectral initialization converges linearly to the true solution, provided $m\ge O((n_1+n_2) r^2 )$ \cite{Ma_2021,charis}.
Compared with the optimal sample complexity $O (\left(n_1+n_2\right)r)$, this requirement is sub-optimal in its dependence on $r^2$.  Moreover,  the iteration complexity scales at least as $O(\kappa \log(1/\epsilon))$ to achieve $\epsilon$-accuracy, which leads to slow convergence when the target matrix is ill-conditioned. Here, $\kappa$ denotes the condition number of the unknown matrix. 

To accelerate convergence for ill-conditioned low-rank matrix recovery, Tong et al.\cite{JMLR:v22:20-1067,Tong2020LowRankMR} proposed the Scaled Gradient Descent (ScaledGD) algorithm, which achieves iteration complexity $O\left(\log(1/\epsilon)\right)$. Nonetheless, its sample complexity remains sub-optimal at $O ((n_1+n_2)r^2)$.  In contrast, St{\"o}ger and Zhu \cite{stoger2025nonconvexmatrixsensingbreaking} made  a major breakthrough by showing that standard GD  with proper spectral initialization enjoys  linear convergence even under  the information-theoretically optimal sample complexity $O (\left(n_1+n_2\right)r)$, but suffer from slower convergence with iteration complexity $O(\kappa^2 \log(1/\epsilon))$.  Furthermore, their guarantees require the target matrix to be positive semidefinite (PSD), which limits its applications. Motivated by these developments, we are led to the following question:

{\em Can scaled gradient descent for low-rank matrix recovery retain an iteration complexity $O(\log(1/\epsilon))$ while simultaneously achieving the optimal sample complexity $O (\left(n_1+n_2\right)r)$?}

 \subsection{Relate work}
 The low-rank matrix recovery problem, which seeks to reconstruct a rank-$r$ matrix $\XXstar \in \R^{n_1\times n_2}$ from a small number of linear measurements $y:=\Aop(\XXstar)\in \mathbb{R}^m$ with $m \ll n_1n_2$, has undergone intensive investigation in recent years. Over the past few decades, numerous algorithms with provable performance guarantees have been developed for this task. A prominent line of work is based on convex relaxation, which replaces the rank function with the nuclear norm $\|\cdot\|_*$  as a convex surrogate, thereby reformulating low-rank matrix recovery as a convex optimization problem. Such convex approaches have been extensively studied in matrix sensing\cite{recht2010guaranteed}, matrix completion\cite{hu2012fast}, and related problems\cite{10.1109/TIT.2013.2294644,10.1145/1970392.1970395}. These methods are known to achieve exact recovery under mild conditions when the number of measurements $m$ scales as $O((n_1+n_2)r)$ up to logarithmic factors, matching the information-theoretic sample complexity. However, because these methods operate over the full matrix space $\mathbb{R}^{n_1 \times n_2}$, their computational cost becomes prohibitive for large-scale problems.

To alleviate this computational burden, another research direction focuses on optimizing the nonconvex objective in \eqref{eq:parametrize problem} using gradient-based methods. For analytical convenience, early works typically introduced explicit regularization terms, such as $\frac{1}{2}\normf{\LL^\top \LL-\RR^\top \RR}^2$ or $\frac{1}{2}\normf{\LL}^2+\frac{1}{2}\normf{\RR}^2$, to balance the norms of the factor matrices $\LL$ and $\RR$ \cite{tu2016lowranksolutionslinearmatrix,park2016nonsquarematrixsensingspurious,Zhu_2018,chen2019noisymatrixcompletionunderstanding} . In sharp contrast, Li et al.\ \cite{Li_2020} demonstrated from a landscape perspective that such balancing regularization is unnecessary, and Ma et al.\ \cite{Ma_2021} showed that unregularized gradient descent converges linearly to the ground-truth matrix provided the initialization is balanced. Similar guarantees were later obtained in related works \cite{du2018algorithmicregularizationlearningdeep,charis}. Although standard gradient descent can theoretically converge to the ground truth, it suffers from a critical limitation: its iteration complexity scales at least linearly with the condition number $\kappa$ of the target matrix. Specifically, it requires $O(\kappa \log(1/\epsilon))$ iterations to reach $\epsilon$-accuracy, leading to slow convergence for ill-conditioned problems. To remedy this, Tong et al.\ \cite{JMLR:v22:20-1067,Tong2020LowRankMR} proposed the Scaled Gradient Descent (ScaledGD) algorithm, combined with spectral initialization, which achieves a fast, condition-number-independent iteration complexity $O\left(\log(1/\epsilon)\right)$. More recently, Jia et al.\ \cite{jia2023preconditioning} demonstrated that ScaledGD with random initialization converges to an $\epsilon$-global minimum in $O(\log(r/\delta)+\log(r/\epsilon))$ iterations, where $\delta$ is a sufficiently small constant. It is also worth noting that the alternating minimization algorithm \cite{jain2013low} and the singular value projection algorithm \cite{netrapalli2014non} enjoy the same $O\left(\log(1/\epsilon)\right)$ convergence rate as ScaledGD, but incur higher per-iteration computational costs: the former solves two least-squares subproblems per iteration, while the latter requires computing the top $r$ singular components of an $n_1\times n_2$ matrix.

Although several  algorithms exhibit fast convergence for ill-conditioned problems, their generally require a sample complexity that  scales at least quadratically in the rank $r$, which is sub-optimal. For instance, ScaledGD \cite{JMLR:v22:20-1067} has sample complexity $O(r^2\kappa^2(n_1+n_2))$, whereas alternating minimization \cite{jain2013low}  requires $O(\left(n_1+n_2\right)r^3 \kappa^4)$ samples. A recent breakthrough by St\"{o}ger and Zhu \cite{stoger2025nonconvexmatrixsensingbreaking} addressed this limitation in the context of low-rank positive semidefinite (PSD) matrix recovery. Using a delicate virtual sequence technique, they showed that vanilla gradient descent, with suitable initialization, achieves the information-theoretically optimal sample complexity $O(r(n_1+n_2)\kappa^2)$, albeit with a slower iteration complexity of $O(\kappa^2 \log(1/\epsilon))$. Building on the techniques developed in \cite{stoger2025nonconvexmatrixsensingbreaking}, a Riemannian Gradient Descent (RGD) method was proposed in \cite{caifast} that attains both the optimal sample complexity $O(r(n_1+n_2)\kappa^2)$ and a fast iteration complexity $O\left(\log(1/\epsilon)\right)$. However, RGD incurs higher memory and computational overhead, since it operates directly in the full matrix space rather than the factor space, and it additionally requires computing projection and retraction operations on a manifold.

In many practical applications, the true rank $r$ of $\XXstar$ is unknown, which naturally leads to studying low-rank recovery in the overparameterized regime. In this setting, one chooses a search rank $k$ for the factorization $\XX=\LL\RR^\top$, with $\LL\in \mathbb{R}^{n_1\times k}$ and $\RR\in \mathbb{R}^{n_2\times k}$,  where $k$ is  strictly larger than the true rank $r$. To ensure global convergence under overparameterization, several extensions of ScaledGD have been proposed. For instance, in the PSD setting,  Zhang et al. \cite{zhang2021preconditioned} generalized ScaledGD by introducing a damping factor $\lambda_t$ to control the singularity of the preconditioning matrix, whereas Xu et al.  \cite{xu2023power} employed a tunable hyperparameter $\lambda$ that remains constant across iterations. For additional developments on overparameterized matrix sensing, the reader is referred to related recent works \cite{geyer2020low,10.5555/3648699.3648862}.

\begin{table}[tp]
 \centering
  \fontsize{10}{16}\selectfont
\caption{Comparison of Non-Convex Methods for Low-Rank Matrix Sensing($n_1=n_2$)}
   \label{tab:method_compare}
  \centering
  \resizebox{\textwidth}{!}{
  \begin{tabular}{c|c|c|c}
    \textbf{Algorithm} & \textbf{Sample Complexity} & \textbf{Iterations} &\textbf{Cost}\\  
    \hline 
    ScaledGD \cite{JMLR:v22:20-1067}  &  $O(n_1r^2\kappa^2)$ & $O(\log\left(1/\ \epsilon\right))$ & $O(n^2_1r)$+$O(r^3)$\\  
    \hline  
    GD (PSD only)  \cite{stoger2025nonconvexmatrixsensingbreaking} & $O(n_1r\kappa^2)$ & $O(\kappa^2 \log\left(1/\ \epsilon\right))$    & $O(n^2_1r)$\\ 
    \hline 
    RGD \cite{caifast} & $O(n_1r\kappa^2)$ &  $O(\log\left(1/\ \epsilon\right))$ & $O(n^2_1r)$+$O(n_1r^2)$+$O(r^3)$\\  
        \hline 
    ScaledGD (this paper) & $O(n_1r\kappa^2)$ &  $O(\log\left(1/\ \epsilon\right))$  & $O(n^2_1r)$+$O(r^3)$\\  
    \hline 
  \end{tabular}
  }
\end{table}

\subsection{Our contributions}
As discussed earlier, almost all existing nonconvex methods based on matrix factorization require a sample complexity of
 $O(r^2\kappa^2(n_1+n_2))$, with the only exception being the work of St\"{o}ger and Zhu \cite{stoger2025nonconvexmatrixsensingbreaking}, who showed that $O(r(n_1+n_2)\kappa^2)$ Gaussian measurements suffice to ensure that vanilla gradient descent enjoys a linear convergence rate. However, their results apply only to PSD matrix sensing, and the step size still depends on the condition number of the low-rank matrix, which leads to slow convergence for ill-conditioned problems.

In this paper, the focus is on the more general problem of recovering an asymmetric low-rank matrix, and it is shown that
$O(r(n_1+n_2)\kappa^2)$ Gaussian measurements are sufficient to guarantee that ScaledGD converges linearly. Moreover, the iteration complexity is $O(\log(1/\epsilon))$ to reach $\epsilon$-accuracy. Compared with existing methods, the proposed result simultaneously achieves three desirable properties: the sample complexity matches the information-theoretic limit, the convergence rate is independent of the condition number of the low-rank matrix, and the per-iteration computational cost is low.  A comparison of several commonly used nonconvex algorithms is summarized in Table \ref{tab:method_compare}. It is worth emphasizing that, although RGD attains the same optimal sample and iteration complexities as the ScaledGD method developed in this paper, it incurs higher memory and computational overhead due to the additional projection and retraction operations on the manifold.

\subsection{Notations}
Throughout this paper, we use $\specnorm{\cdot}$ and $\normf{\cdot}$ to denote the operator norm and Frobenius norm of a matrix, respectively, and  $ \twonorm{\bm{v}}$  to denote the Euclidean norm of a vector $\bm{v}$.
The condition number of the true matrix $\XXstar$ is defined as
\begin{equation*} 
    \kappa:= \frac{\specnorm{\XXstar}}{\sigma_{\min} (\XXstar)},
\end{equation*}
where $\sigma_{\min} (\XXstar)$ is the smallest nonzero singular value of $\XXstar$.  The compact singular value decomposition (SVD) of $\XXstar$ is given by $\XXstar={\VV}_{\star} \Sg_{\star} {\WW}_{\star}^\top$, where  ${\VV}_{\star} \in \R^{n_1\times r}$,  ${\WW}_{\star} \in \R^{n_2\times r}$, and $\SSigma_{\star} \in \R^{r\times r}$ is a diagonal matrix whose diagonal entries are the singular value of $\XXstar$,  arranged in  non-increasing order.  Moreover, let ${\VV}_{\star, \perp} \in \R^{n_1 \times (n_1-r)}$ ( resp. ${\WW}_{\star, \perp} \in \R^{n_2 \times (n_2-r)}$)  denote  matrices  whose columns form orthonormal basis for the orthogonal complements of the  column spaces of ${\VV}_{\star}$ ( resp. ${\WW}_{\star}$).

\subsection{Organnization} \label{Organnization}
The rest of the paper is organized as follows. Section~\ref{sec:algorithms} introduces the proposed ScaledGD algorithm. In Section~\ref{sec:main result}, we present our main theoretical result, Theorem~\ref{thm:main}. Section~\ref{sec:proof main result} contains the proof of the main theorem, including the construction of virtual sequences and the key lemmas, while most technical details are deferred to the Appendix. Section~\ref{sec:experiment} illustrates the performance of ScaledGD on low-rank matrix recovery problems. Finally, Section~\ref{sec:discussions} concludes the paper with  a discussion of potential directions for future research.

\section{Scaled Gradient Descent}\label{sec:algorithms}
The program we consider is
\begin{equation}   \label{eq:parametrize}
    \min_{\LL\in \mathbb{\R}^{n_1\times r},\RR \in \mathbb{\R}^{n_2\times r}} \mathcal{L} {\left(\LL, \RR\right)}:=\frac{1}{2}\|\Aop\left(\LL \RR^\top\right)-\yy \|_2^2, \notag
\end{equation}
where $\Aop: \mathbb{R}^{n_1 \times n_2}\rightarrow \mathbb{R}^m$ is  linear operator  defined by
\begin{equation} \label{eq:obsevation}
 \left[\Aop (\XX) \right]_i :=  \frac1{\sqrt m} \nj{\AAi, \XX}, \quad  i=1,2,\ldots, m,  
\end{equation}
and $\yy:=\Aop(\XXstar)\in \mathbb{R}^m$ with $\rank(\XXstar)=r$. To solve it, we apply the scaled gradient descent developed in \cite{JMLR:v22:20-1067}, whose iteration updates are given by
\begin{align} \label{equ:gradientdescentdefinition}
    \qquad \qquad \qquad \LL_{t+1}:=&\LL_t - \mu \nabla_{\LL_t} \mathcal{L} \left(\LL_t, \RRt\right)\left(\RRtT\RRt\right)^{-1},\notag \\
    \RR_{t+1}:=&\RR_t - \mu \nabla_{\RR_t} \mathcal{L} \left(\LL_t, \RRt\right)\left(\LLtT\LLt\right)^{-1},
\end{align} 
where $\mu >0$ denotes the step size. A direct calculation shows that 
\[
    \nabla_{\LL_t} \mathcal{L} \left(\LL_t, \RRt\right) = \Aop^* \left( \Aop\left(\XXt\right)-\yy \right)   \RRt ,
 \]
    and 
    \[
        \nabla_{\RR_t} \mathcal{L} \left(\LL_t, \RRt\right)= \Aop^* \left( \Aop\left(\XXt\right) -\yy \right)^\top \LLt,
    \]
where $\XX_t:=\LL_t \RR_t^\top$, and $\Aop^*: \R^m \rightarrow \R^{n_1\times n_2}$ is the adjoint operator of $\Aop$ defined by
\begin{align*}
    \Aop^* (\vv) = \sum_{i=1}^m \vv_i \AAi  \quad \mbox{for any} \quad \vv\in \R^m.
\end{align*}
Due to the non-convexity of the problem, a good initialization $(\LL_0, \RR_0)$ is crucial. We adopt the  spectral method used in \cite{stoger2025nonconvexmatrixsensingbreaking,Tong2020LowRankMR,JMLR:v22:20-1067}. Specifically, let the top-$r$  singular value decomposition of  $ \Aop^* (\yy)$ be
\begin{align*}
    \Aop^* (\yy)
    &= \widetilde\VV\widetilde\SSigma \widetilde\WW^\top,
\end{align*}
where $\widetilde{\VV} \in \R^{n_1 \times r}$ and $\widetilde{\WW} \in \R^{n_2 \times r}$ contain the top-$r$ left and right singular vectors, respectively, and $\widetilde\SSigma \in \R^{r \times r}$ is a diagonal matrix with the corresponding singular values arranged in nonincreasing order. Then the initial guess $(\LL_0, \RR_0)$ is then chosen as
\begin{align*}
    \LL_0 =\widetilde \VV{\widetilde\SSigma}^{1/2}, \quad \mbox{and} \quad \RR_0=\widetilde \WW{\widetilde\SSigma}^{1/2}.
\end{align*}
The scaled gradient descent algorithm with spectral initialization is summarized in Algorithm \ref{algorithm:SGD}.
\begin{algorithm}
    \caption{Scaled Gradient Descent (ScaledGD) for Low-Rank Matrix Recovery}
    \label{algorithm:SGD}
    \begin{algorithmic}
    \STATE{\textbf{Input:}
     Measurement operator $\mathcal{A}: \R^{n_1 \times n_2}\to \R^{m}$,
     observations $\yy \in \R^m$, step size $\mu>0$, the number of iteration $T$.
    }
    \STATE{
        \textbf{Spectral Initialization:}
        Let  $ \widetilde\VV{\widetilde\SSigma} \widetilde\WW^\top $ be the top-$r$ SVD of  $  \mathcal{A}^* (\yy)$.  Define  $ \LL_0 :=  \widetilde \VV{\widetilde\SSigma}^{1/2} $ and $\RR_0=\widetilde \WW{\widetilde\SSigma}^{1/2}$.
    }
    \STATE{
        \textbf{Iteration:}
        \FOR{$t = 0, 1, 2,\ldots, T-1$}
        \STATE{
           \begin{align*}
               \LL_{t+1}  =&
    \LLt 
    + \mu \Aop^* \left(\yy- \Aop\left(\XXt\right) \right)   \RRt \left(\RRtT\RRt\right)^{-1}, \\
    \RR_{t+1} 
    =&
    \RRt + \mu \Aop^* \left( \yy-\Aop\left(\XXt\right) \right)^\top \LLt \left(\LLtT\LLt\right)^{-1}.
           \end{align*}
        }
        \ENDFOR
    }
     \STATE{\textbf{Output:}} $\XX_T:=\LL_{T}\RR_{T}^\top$.
    \end{algorithmic}
\end{algorithm}

\section{ Main result} \label{sec:main result}
In this section, we demonstrate that the ScaledGD for  low-rank matrix recovery converges linearly while achieving  the optimal sample complexity $O((n_1+n_2)r)$. In particular,  its  iteration complexity is $O(\log(1/\epsilon))$ to reach an $\epsilon$-accuracy solution. 
To begin, let $\VXX \SSigma_{\star} \WXXT$ denote the compact singular value decomposition of true matrix $\XXstar$, and define 
\begin{equation*}
    \LL_{\star}:=\VXX\SSigma_{\star}^{1/2},  \quad  \quad \RR_{\star}:=\WXX\SSigma_{\star}^{1/2}.
\end{equation*}
Note that for any invertible matrix $\QQ \in \R^{r\times r}$, one can write $\XXstar=\left(\LL_{\star}\QQ\right)\left(\RR_{\star}\QQ^{-\top}\right)^\top$. Therefore, to measure the discrepancy between the $t$-th iteration $(\LL_t, \RR_t)$ and the true factors $( \LL_{\star},  \RR_{\star})$, we adopt the following metric \cite{JMLR:v22:20-1067}:
\begin{equation}\label{eq:def_dist}
    {\dist}^2 \left( \XXt, \XXstar \right)
    :=
    \underset{\QQ \in \text{GL}(r)}{\inf}
    \normf{
        \left(\LLt \QQ - \LL_{\star} \right)\mathbf\SSigma^{1/2}_{\star}}^2
        +\normf{
        \left(\RRt \QQ^{-\top} - \RR_{\star} \right)\mathbf\SSigma^{1/2}_{\star}}^2,
\end{equation}
where $\XXt:=\LLt \RRtT \in \R^{n_1 \times n_2}$ and GL$(r)$ denotes the set of all invertible $r\times r$ matrices.
With this notation in place, our main result is as follows:

\begin{theorem}\label{thm:main}
Let $\XXstar \in \R^{n_1 \times n_2}$ with $\rank(\XXstar)=r$,  and let $\boldsymbol{A}_1,\ldots,\boldsymbol{A}_m \in \mathbb{R}^{n_1 \times n_2}$ be Gaussian random matrices with i.i.d. entries distributed as $ \mathcal{N} \left( 0,1 \right) $. Assume that  $\mathcal{A}: \R^{n_1 \times n_2} \to \R^m$ is the linear operator defined in \eqref{eq:obsevation}. Let $\{(\LL_t, \RR_t)\}_{t\ge 0}$  be the sequence generated by Algorithm \ref{algorithm:SGD} with $\yy = \Aop \left( \XXstar \right) \in \R^m $ and step size $c_{\mu} \le \mu \le \frac{1}{32}$. Then,  with probability at least $1-7\exp \left( -\left(n_1+n_2\right) \right)$, 
\begin{equation} \label{eq:maindist}
        \dist \left( \XXt, \XXstar \right) \le
      8\sqrt{r}\left(1-\frac{\mu}{10} \right)^{t}c_0 \sigma_{\min} (\XXstar),
    \end{equation}
and 
 \begin{align}
    \normf{\XXt-\XXstar} \le 12\sqrt{r}\left(1-\frac{\mu}{10} \right)^{t} c_0 \sigma_{\min} (\XXstar)
\end{align}
hold for all iterations $ t\ge 0$,  provided $ m \ge C \left(n_1+n_2\right) r \kappa^2$.  Here, $\XXt=\LLt \RRtT \in \R^{n_1 \times n_2}$, $\dist(\cdot)$ is defined in \eqref{eq:def_dist},  $C,c_0$ are absolute constants, and $c_{\mu}>0$ is a sufficiently small constant.
\end{theorem}

\begin{remark}
Theorem \ref{thm:main} implies that after $O\left(\frac{\log(\sqrt{r}/\epsilon)}{\mu}\right)$ iterations, ScaledGD  satisfies $\dist(\XXt, \XXstar) \le \epsilon \sigma_{\min}(\XXstar)$. In particular, the step size  is independent of the condition number $\kappa$ of $\XXstar$, which leads  to a fast convergence rate for ill-conditioned matrices.  Moreover, the result achieves the optimal sample complexity $O ((n_1+n_2) r \kappa^2)$ and applies to  the general asymmetric matrix setting, rather than being restricted to the PSD case \cite{stoger2025nonconvexmatrixsensingbreaking}.
\end{remark}

\section{Proof of the main result}\label{sec:proof main result}
In this section, we present the proof of the main result. Throughout, we assume that $\boldsymbol{A}_1,\ldots,\boldsymbol{A}_m \in \mathbb{R}^{n_1 \times n_2}$ are Gaussian random matrices with i.i.d. entries drawn from $ \mathcal{N} \left( 0,1 \right) $. We begin by recalling the Restricted Isometry Property (RIP)\cite{recht2010guaranteed,CandesRecht_2012,CandesRecht_2012,candes2011tight}, which plays a central  role in the analysis of low-rank matrix recovery problems.

\begin{definition}[RIP]\label{definition:RIP}
  A linear measurement operator $\Aop: \mathbb{R}^{n_1 \times n_2} \to \R^m$ 
   is said to satisfy the rank-$r$ RIP with a constant $\delta_r\in(0,1)$, if 
   \begin{equation*} 
    \left(1-\delta_r\right) \normf{\ZZ}^2
    \le 
    \twonorm{\mathcal{A} (\ZZ)}^2
    \le 
    \left(1+\delta_r\right) \normf{\ZZ}^2
   \end{equation*} 
   holds for all  $\ZZ \in \R^{n_1\times n_2}$ with  $\rank(\ZZ)\le r$.
\end{definition}

\begin{lemma}\cite{candes2011tight}\label{lem:rank_RIP}
Let $\boldsymbol{A}_1,\ldots,\boldsymbol{A}_m \in \mathbb{R}^{n_1 \times n_2}$ be Gaussian random matrices with i.i.d. entries distributed as $ \mathcal{N} \left( 0,1 \right) $, and assume $\mathcal{A}: \R^{n_1 \times n_2} \to \R^m$ is the linear operator defined in \eqref{eq:obsevation}. Then, for any $0<\varepsilon<1$, with probability $1-\varepsilon$, the operator $\mathcal{A}$ satisfies rank-$r$ RIP with constant $\delta_r$, provided 
\begin{align}\label{eq:lowerbound_m}
    m\geq C\delta_{r}^{-2} (r\left(n_1+n_2\right)+\log(2\varepsilon^{-1})),
\end{align}
where $C>0$ is a universal constant. In particular, if $m\geq C\delta_r^{-2} r\left(n_1+n_2\right)$, then with probability at least $1-\exp(-\left(n_1+n_2\right))$,  the operator $\mathcal{A}$ satisfies the rank-$r$ RIP with constant $\delta_r$.
\end{lemma}

\subsection{The main idea of the proof}
Under a mild RIP condition, prior local convergence theory \cite[Lemma 14]{JMLR:v22:20-1067} shows that once  $\text{dist} \left( \XX_T, \XXstar \right)  \le  0.1\sigma_{\min} (\XXstar)$ holds for some $T\ge 0$, ScaledGD then converges linearly to the true matrix $\XXstar$, as stated below. 

\begin{lemma}{\cite[Lemma 14]{JMLR:v22:20-1067}} \label{lemma:Phase2}
   Assume that the measurement operator $\Aop$ defined in \eqref{eq:obsevation} satisfies rank $2r$ RIP   with constant $\delta_{2r}\le 0.02$. Let $\{(\LL_t, \RR_t)\}_{t\ge 0}$  be the  sequence generated by Algorithm \ref{algorithm:SGD} with  $\yy = \Aop \left( \XXstar \right) \in \R^m $ and step size $0 <\mu \le 2/3$. If 
   \begin{equation}\label{ineq:closenesscondition}
    \dist \left( \XX_T, \XXstar \right)  \le  0.1\sigma_{\min} (\XXstar)
   \end{equation} 
   for some iteration number $T\ge 0$, then for all $t \ge T$ it holds  that 
   \begin{equation*}
    \dist \left( \XXt, \XXstar \right)
    \le  \left(1- 0.6\mu\right)^{t-T}
    \dist(\XX_T, \XXstar)
   \end{equation*}
   and 
   \[
   \normf{\XX_t-\XXstar} \le 1.5  \dist\left( \XXt, \XXstar \right).
   \]
\end{lemma}
According to Lemma \ref{lemma:Phase2}, to guarantee the linear convergence of ScaleGD, it suffices to verify that $\Aop$ satisfies the rank-$2r$ RIP with constant $\delta_{2r}\le 0.02$, and that  condition \eqref{ineq:closenesscondition} holds for some $T>0$. Lemma \ref{lem:rank_RIP} shows that  under Gaussian design and when $m \ge O(\delta_{2r}^{-2}(n_1+n_2)r)$, the operator $\Aop$ satisfies the rank-$2r$ RIP with constant $\delta_{2r}$ with high probability. Hence, the requirement $\delta_{2r} \le 0.02$ can be easily met with optimal sample complexity.

The main difficulty lies in ensuring that \eqref{ineq:closenesscondition} holds for some $T>0$ under the sample complexity $m=O((n_1+n_2)r)$. Existing works such as \cite{JMLR:v22:20-1067} simply take $T=0$ and use spectral initialization to guarantee \eqref{ineq:closenesscondition}, at the price of requiring $m\ge O((n_1+n_2)r^2\kappa^2)$. In particular, under spectral initialization, Lemma 15 of \cite{JMLR:v22:20-1067} shows that
\begin{equation} \label{eq:distx0x}
 \text{dist} \left( \XX_0, \XXstar \right)  \le \sqrt{\sqrt{2}+1}  \normf{\XX_0 - \XXstar} \le 3 \sqrt{r}  \twonorm{\XX_0 - \XXstar} \le 6 \delta_{2r} \sqrt{r}  \kappa \sigma_{\min} (\XXstar).
\end{equation}
Therefore, to ensure \eqref{ineq:closenesscondition}, one needs $\delta_{2r} \le O(1/(\kappa \sqrt{r}))$, which in turn forces  $m\ge O((n_1+n_2)r^2\kappa^2)$, a sub-optimal sample complexity.

In this paper,  we aim to show that \eqref{ineq:closenesscondition} still holds even when $m=O((n_1+n_2)r\kappa^2)$.  From Lemma \ref{lemma:spectralinitialization}, under this sample complexity $m=O((n_1+n_2)r\kappa^2)$, we obtain
\[
\twonorm{\XX_0 - \XXstar} \le c_0 \sigma_{\min} (\XXstar).
\]
If, in addition,  we can establish a linear contraction in the operator norm, namely,
\begin{equation} \label{eq:2normcont}
\twonorm{\XX_t - \XXstar} \le (1-\rho)^{t} \twonorm{\XX_{0} - \XXstar}
\end{equation}
for some constant $0<\rho<1$, then after $T:=O(\log(\sqrt r))$ iterations we obtain 
\[
\twonorm{\XX_T- \XXstar} \lesssim \sigma_{\min} (\XXstar)/\sqrt{r}.
\]
Arguing as in \eqref{eq:distx0x}, this implies
\[
 \text{dist} \left( \XX_T, \XXstar \right)  \le \sqrt{\sqrt{2}+1}  \normf{\XX_T - \XXstar} \le 3 \sqrt{r}  \twonorm{\XX_T - \XXstar} \lesssim \sigma_{\min} (\XXstar),
\]
which yields condition \eqref{ineq:closenesscondition}. Applying Lemma~\ref{lemma:Phase2} then gives linear convergence of ScaledGD with optimal sample complexity.

Establishing \eqref{eq:2normcont}, however, is itself nontrivial. Since the contraction is in the operator norm rather than the Frobenius norm, a more delicate error analysis is required, and standard arguments based solely on RIP no longer suffice (see Section~\ref{sec:virtualsq} for details). To overcome this, we employ the decoupling technique based on virtual sequences developed by St\"{o}ger and Zhu in \cite{stoger2025nonconvexmatrixsensingbreaking}, and show that after
\begin{equation} \label{eq:defT}
    T :=   \Big\lceil \frac{10}{\mu } \log \left(  10\sqrt{r}  \right)     
    \Big\rceil
\end{equation}
iterations, the iterate $\XX_T$ generated by  ScaledGD satisfies condition \eqref{ineq:closenesscondition}. Here, $\mu$ is the step size and $r$ is the rank of the target matrix $\XXstar$.

In summary, the proof of the main theorem proceeds in three steps:
\begin{itemize}
\item[(i)] Introduce a refined decoupling framework based on virtual sequences and derive several sharp error bounds (Subsection~\ref{sec:virtualsq}).
\item[(ii)] Use these virtual sequences to establish convergence of $\XX_t$ to $\XXstar$ in operator norm (Subsection~\ref{subsection:mainconvergencelemma}).
\item[(iii)] Combine these ingredients to complete the proof of the main result (Subsection~\ref{secproofmainresult}).
\end{itemize}

\subsection{Virtual sequences} \label{sec:virtualsq}
A key step in establishing contraction of $\twonorm{\XX_t - \XXstar}$ in the operator norm is to show 
\begin{equation} \label{eq:2norm}
\specnorm{ \left(\Aops-\IdOp\right) (\XXstar-\XXt)} \lesssim \specnorm{\XXstar-\XXt}.
\end{equation}
Standard arguments based on the RIP yield only
\begin{equation} \label{eq:Fnorm}
\specnorm{ \left(\Aops-\IdOp\right) (\XXstar-\XXt)} \lesssim \normf{\XXstar-\XXt} \le \sqrt{2r}\specnorm{\XXstar-\XXt},
\end{equation}
where the last inequality is due to $\mbox{rank}(\XXstar-\XXt)\le 2r$. This leaves an undesirable $O(\sqrt{r})$ gap  between \eqref{eq:2norm} and \eqref{eq:Fnorm}.  To close this gap, the operator norm is estimated directly via an $\epsilon$-net argument combined with a decoupling technique. Let ${\mathbb S}^{n-1}:= \left\{ \xx \in \R^{n}: \twonorm{\xx} =1 \right\}$ be the unit sphere in $\R^n$. There  exists a $1/4$-net 
\begin{equation} \label{eq:fixedNep}
\mathcal{N} \subset \bS^{n_1-1} \times  \bS^{n_2-1} 
\end{equation}
with the cardinality $|{\mathcal N} | \le 12^{n_1+n_2}$.  For any matrix $\ZZ\in \R^{n_1\times n_2}$,  by a standard $\epsilon$-net bound on the spectral norm \cite[Excise 4.4.3]{HDP},
\begin{equation*}
    \specnorm{\ZZ}=\underset{\left(\ww, \vv\right) \in \bS}{\text{sup}} \nj{\ww \vv^\top, \ZZ} \le 2\underset{\left(\ww, \vv\right) \in \mathcal{N}}{\text{sup}} \nj{\ww \vv^\top, \ZZ},
\end{equation*}
where $\bS:=\bS^{n_1-1} \times  \bS^{n_2-1} $. Hence,
\[
\specnorm{ \left(\Aops-\IdOp\right) (\XXstar-\XXt)} \le \frac2m \underset{\left(\ww, \vv\right) \in \mathcal{N}}{\text{sup}} \sum_{i=1}^m  \xkh{\nj{\AAi,\XXstar-\XXt} \nj{\AAi,\ww \vv^\top}-\nj{\XXstar-\XXt,\ww \vv^\top}}
\]
Due to the stochastically dependent between the iteration $\XXt$ and $\dkh{\nj{\AAi,\ww \vv^\top}}_{i=1}^m$,  an expectation-type bound  $O(\specnorm{\XXstar-\XXt})$ cannot be obtained directly. To decouple this dependence, for each pair $\left(\ww, \vv\right) \in \mathcal{N}$,  define
\[
 \Projwperp (\AAi):= \AAi- \nj{\AAi,\ww\vv^\top}\ww\vv^\top.
\]
Since each $\AAi$ is a standard Gaussian random matrix, the family $\dkh{\Projwperp (\AAi)}_{i=1}^m$ is independent of 
 $\dkh{\nj{\AAi,\ww \vv^\top}}_{i=1}^m$.  Therefore, we can use $\dkh{\Projwperp (\AAi)}_{i=1}^m$ to construct an auxiliary virtual sequence $\dkh{\XX_{t}^{(\ww,\vv)}}_{t\ge 0}$ such that they are  stochastically independent of $\dkh{\nj{\AAi,\ww \vv^\top}}_{i=1}^m$ and sufficiently close to the original sequence $\dkh{\XX_{t}}_{t\ge 0}$. More precisely,  define a virtual measurement operator $\Aopw : \R^{n_1 \times n_2} \rightarrow \R^{m+1}  $ by
\[
 [\Aopw (\ZZ)]_i := \left\{ \begin{array}{ll}  \frac{1}{\sqrt{m}} \nj{\AAiw, \ZZ}, & \mbox{for} \quad i \in [m] \\
  \nj{\ww \vv^\top, \ZZ}, & \mbox{for} \quad i=m+1. \end{array} \right.
\]
Here, $\AAiw := \Projwperp (\AAi)$, and  the $(m+1)$-th coordinate records the component in the direction spanned by  $\ww \vv^\top$. Then,  following the same procedure as Algorithm \ref{algorithm:SGD},  for each pair $\left(\ww, \vv\right) \in \mathcal{N}$, a virtual sequence $\dkh{\XX_{t}^{(\ww,\vv)}}_{t\ge 0}$ is generated by replacing $\mathcal A$ with $\Aopw$. The details are summarized in Algorithm \ref{algorithm:SGDvirtual}.  By construction,  the entire sequences  $ \left( \LLtw \right)_{t\ge0}$ and $ \left( \RRtw \right)_{t\ge 0}$  are stochastically independent of $ \left( \nj{\AAi, \ww \vv^\top} \right)_{i=1}^m$, which is the key decoupling property used in the subsequent analysis.

\begin{algorithm}
    \caption{The virtual sequence corresponds to $\left(\ww, \vv\right)$}
    \label{algorithm:SGDvirtual}
    \begin{algorithmic}
    \STATE{\textbf{Input:}
     Measurement operator $\Aopw: \R^{n_1 \times n_2}\to \R^{m+1}$, step size $\mu>0$, the number of iteration $T$.
    }
    \STATE{
        \textbf{Spectral Initialization:}
        Let  $\widetilde\VV^{(\ww,\vv)}\widetilde\LLambda^{(\ww,\vv)} \left({\widetilde\WW^{(\ww,\vv)}}\right)^\top$ be the top-$r$ singular value decomposition of  $  \Aopws \left(\XXstar \right)$.  Then define  $ \LL_{0}^{(\ww,\vv)}:={\widetilde \VV}_{r}^{(\ww,\vv)}\left({{\widetilde\LLambda}_{r}^{(\ww,\vv)}}\right)^{1/2}$ and $\RR_{0}^{(\ww,\vv)}:={\widetilde \WW}_{r}^{(\ww,\vv)}\left({\widetilde\LLambda_{r}^{(\ww,\vv)}}\right)^{1/2}$.
    }
    \STATE{
        \textbf{Iteration:}
        \FOR{$t = 0, 1, 2,\ldots, T-1$}
        \STATE{
           \begin{align*}
              \LL_{t+1}^{(\ww,\vv)}:= &  \LLtw + \mu \Aop^* \left( \yy-\Aop \left(\XXtw\right)\right)  \RRtw\left(\RRtwT\RRtw\right)^{-1} \\
    \RR_{t+1}^{(\ww,\vv)}
:= &
\RRtw + \mu \Aop^* \left( \yy-\Aop \left(\XXtw\right)\right)^\top \LLtw\left(\LLtwT\LLtw\right)^{-1}.  
           \end{align*}
        }
        \ENDFOR
    }
     \STATE{\textbf{Output:}} $\dkh{\XX_{t}^{(\ww,\vv)}:=\LL_{t+1}^{(\ww,\vv)} {\RR_{t+1}^{(\ww,\vv)}}^\top}_{t\ge 0}$.
    \end{algorithmic}
\end{algorithm}

With the help of virtual sequences, the quantity $\specnorm{ \left(\Aops-\IdOp\right) (\XXstar-\XXt)}$ can be tightly controlled, as shown in \cite[Lemmas 4 and 5]{caifast}.

\begin{lemma} \cite{caifast} \label{lemma:key}
Let $\mathcal{N}$ be as in \eqref{eq:fixedNep} and let $\dkh{\XX_{t}^{(\ww,\vv)}}_{t\ge 0}$ be the virtual sequence constructed for each $\left(\ww,\vv\right) \in \mathcal{N}$ by  Algorithm \ref{algorithm:SGDvirtual}. Then, with probability at least $1-2\exp \left(-2\left(n_1+n_2\right) \right)$, for all $\left(\ww,\vv\right) \in \mathcal{N}$ and all $0\le t\le T-1$,
    \begin{equation*} \label{eq: front key idea}
    \vert   \nj{\ww \vv^\top, \left(\Aops \right) \left( \Projwperp \left( \XXstar - \XXtw  \right)  \right)  }\vert \\
    \le 
    4 \sqrt{\frac{n_1+n_2}{m}} \twonorm{ \Aop \left( \Projwperp \left( \XXstar - \XXtw \right) \right) }.
    \end{equation*}
    Here, $T$ is defined in \eqref{eq:defT}. In particular, if $\mathcal{A}$ satisfies RIP of order $4r+1$ with constant $\delta$,  then for all $\left(\ww,\vv\right) \in \mathcal{N}$ and $0\le t\le T-1$,
        \begin{equation*}   \label{eq:key idea}
        \specnorm{ \left( \Aops -\IdOp \right) \left( \XXstar - \XXt \right)  }
        \le  4c^{\prime} \specnorm{ \XXstar - \XXt }+  6c^{\prime} \supw  \normf{\XXt-\XXtw},    
     \end{equation*}
 where $c^{\prime}:=\max \left\{ \delta; 8 \sqrt{2r\left(n_1+n_2\right)/m} \right\}$. 
\end{lemma}

Let the compact singular value decomposition of the true matrix $\XXstar$ be $\XXstar=\VXX \SSigma_{\star} \WXXT$, where the diagonal entries of $\SSigma_{\star} \in \R^{r \times r}$ are  arranged in non-increasing order. Similarly, for each $t\ge 0$, let $\XXt={\VV}_{t} \Sg_{t} {\WW}_{t}^\top$ be the compact singular value decomposition of $\XXt$.
The next lemma shows that,  for each $\left(\ww,\vv\right) \in \mathcal{N}$,  under suitable conditions, if the  virtual sequence $\dkh{\XX_{t}^{(\ww,\vv)}}_{t\ge 0}$ is close to $\dkh{\XX_t}_{t\ge 0}$ at $t$-th iteration, then the sum of the  projections of $\XX_{t+1}-\XXtwplus$ onto the column and row spaces of $\XXstar$ admits a sharp upper bound.

\begin{lemma}\label{lemma:auxsequencecloseness}
For any constants $\constone \le \frac{1}{20}$ and $0<\consttwo, \constthree \le \frac{1}{360}$, suppose that 
\begin{align}
    \max\{ \specnorm{\VXXPT \VLLt}, \specnorm{\WXXPT \VRRt}\} 
    &\le 
   \constone, \label{assump:closeness5} \\
    \specnorm{ \XXt - \XXstar}
    &\le
    \consttwo \sigma_{\min} (\XXstar),\label{assump:closeness7}\\
    \normf{ \XXt - \XXtw }
    &\le 
    \constthree \sigma_{\min} \left(\XXstar \right).\label{assump:closeness8}
\end{align}
In addition, assume that the conclusion of Lemma \ref{lemma:key} holds and the  step size satisfies $\mu \le \frac{1}{32} $. Then,
\begin{align}
    &\normf{\VXXT \left( \XXtplus - \XXtwplus \right)}+\normf{\left( \XXtplus - \XXtwplus \right)\WXX}\\
    &\le 
    \left(1-\frac{\mu}{4}\right)\left(\normf{\VXXT\left(\XXt-\XXtw\right)}+\normf{\left(\XXt-\XXtw\right)\WXX}\right)+\frac{1}{9}\mu \specnorm{\XXstar-\XXt}. \notag
\end{align}
Moreover, it holds
\begin{equation}
    \normf{\XXtplus-\XXtwplus} \le  \frac{\sigma_{\min} \left(\XXstar \right)}{80} .
    \label{eq:advance estimate1}
\end{equation}
\end{lemma}
\begin{proof}
See Section \ref{subsec:proofauxsequencecloseness}.
\end{proof}

\subsection{Error contraction} \label{subsection:mainconvergencelemma}
Building on the virtual sequences, this section shows that the iterates $\XX_t $ converges to $\XXstar$ in the operator norm. The next lemma establishes a linear contraction for the sum of the projections of  $\XXstar-\XX_t$ onto the column and row spaces of $\XXstar$.

\begin{lemma}\label{lemma:localconv}
    For any constants $0<c_2,c_5 \le 0.01$, assume that
    \begin{align}
        \max\{ \specnorm{\VXXPT \VLLt}, \specnorm{\WXXPT \VRRt}\} 
        &\le 
        \frac{1}{8}, \label{ineq:localconv8} \\
        \specnorm{\XXstar - \XXt} 
        &\le \consttwo \sigma_{\min} (\XXstar), \label{ineq:localconv5}\\
         \specnorm{\left(\Aops - \IdOp \right) \left( \XXstar - \XXt \right)}
        &\le 
        \constfive \sigma_{\min} \left( \XXstar \right),\label{ineq:localconv4}
    \end{align}
 and that the step size satisfies  $ \mu \le \frac{1}{15} $. Then 
    \begin{equation}
        \specnorm{\XXtplus-\XXstar} \le \frac{\sigma_{\min} (\XXstar)}{80} 
        \label{eq:advance estimate2}
    \end{equation}
and
   \begin{align*}
            \specnorm{
                \VXXT \left( \XXstar - \XXtplus \right)}+&
                \specnorm{
                \left( \XXstar - \XXtplus \right)\WXX} \notag \\
                \le& 
                 \left( 1 - \frac{\mu  }{4} \right)\left(
            \specnorm{\VXXT\left( \XXstar - \XXt \right)}
            +\specnorm{\left( \XXstar -\XXt \right)\WXX}\right)
            +
            6 \mu  \specnorm{\EEb_t }.
        \end{align*} 
   Here, $\EEb_t :=\left( \Aops - \IdOp \right) \left(\XXstar - \XXt \right)$.     
\end{lemma}
\begin{proof}
See Appendix \ref{sec:prooflocalconv_b}.
\end{proof}

For $t\in \N$, define
\begin{equation} \label{eq:Gt0}
 \GG_t:=  \specnorm{ \XXstar - \XXt }+  \supw  \normf{\XXt - \XXtw},
\end{equation}
and 
    \begin{align} \label{eq:Gtstar0}
        \GG_{t,\star}:&= \specnorm{\VXXT\left(\XXstar - \XXt \right)}+\specnorm{\left(\XXstar - \XXt \right)\WXX}
        + \supw  \normf{\VXXT\left(\XXt - \XXtw\right)} \notag\\
        &+\supw  \normf{\left(\XXt - \XXtw\right)\WXX}. 
    \end{align}
    
Based on  Lemma~\ref{lemma:auxsequencecloseness} and Lemma Lemma~\ref{lemma:localconv}, we can obtain the following result, which shows that both $\GG_t$ and $ \GG_{t,\star}$ contracts linearly.

\begin{lemma}  \label{Le:4_6}
 Assume that  $0<c_0 \le \frac{1}{1080}$ is a universal constant, and  $m \ge C \left(n_1+n_2\right)r \kappa^2$ for some constant $C>0$. Then,  with probability at least $1- 6\exp(-\left(n_1+n_2\right))$, 
\begin{equation}  \label{assumption:initiation}
    \GG_{0,\star}\le 2\GG_0\le 4c_0 \sigma_{\min} (\XXstar).
\end{equation}
 Moreover, if  the step size satisfies $\mu \le \frac{1}{32}$, then for all $0\le t\le T$,
\begin{equation} \label{eq:Gtstar}
    \GG_{t,\star}\le 2\left(1-\frac{\mu}{10} \right)^t c_0 \sigma_{\min} (\XXstar)
\end{equation}
and
\begin{align} \label{eq:Gt}
    \GG_t \le 3\left(1-\frac{\mu}{10} \right)^{t} c_0 \sigma_{\min} (\XXstar).
\end{align}
Here, $T$ is defined in \eqref{eq:defT}.
\end{lemma}
\begin{proof}
See Section \ref{pf:Le4.6}.
\end{proof}

\subsection{Proof of Theorem \ref{thm:main}}\label{secproofmainresult}
Now all the ingredients are in place to prove the main result.
\begin{proof}[Proof of Theorem \ref{thm:main}]
Since $\bm{A}_1,\ldots,\bm{A}_m \in \mathbb{R}^{n_1 \times n_2}$ are Gaussian random matrix with i.i.d. $ \mathcal{N} \left( 0,1 \right)$ entries,  Lemma  \ref{lem:rank_RIP} implies that  when $m \gtrsim r\left(n_1+n_2\right) \kappa^2$, the measurement operator $\mathcal{A}$ satisfies RIP of order $\left(4r+1\right)$ with a constant $\delta=\delta_{4r+1} \le c^{\prime}$ with probability $1-\exp(-\left(n_1+n_2\right))$, where $0<c_0, c^{\prime} \le \frac{1}{1080}$. 
Set
\begin{equation*}
    T:=
   \Big\lceil \frac{10}{\mu } \log \left(  10\sqrt{r}  \right)  \Big\rceil.
\end{equation*} 
According to Lemma~\ref{Le:4_6}, when $m \ge C \left(n_1+n_2\right)r \kappa^2$ for some constant $C>0$, with probability at least $1- 6\exp(-\left(n_1+n_2\right))$, 
\[
 \GG_t \le  3\left(1-\frac{\mu}{10} \right)^{t} c_0 \sigma_{\min} (\XXstar)
\]
holds for all $0\le t\le T$.  From the definition of $\GG_t$ in \eqref{eq:Gt0}, this yields
\[
\specnorm{ \XXstar - \XXt } \le 3\left(1-\frac{\mu}{10} \right)^{t} c_0 \sigma_{\min} (\XXstar).
\]
Hence
\begin{eqnarray}
 \dist \left( \XXt, \XXstar \right)   & \le &  \sqrt{\sqrt{2}+1}  \normf{\XXt - \XXstar}  \notag\\
 &\le &  \sqrt{2r} \cdot \sqrt{\sqrt{2}+1}  \specnorm{\XXt - \XXstar}  \notag \\
 &\le & 8\sqrt{r}\left(1-\frac{\mu}{10} \right)^{t}c_0 \sigma_{\min} (\XXstar), \notag
\end{eqnarray}
where the first inequality follows from Lemma~\ref{lemma:procrustebound} and the second uses $\rank(\XXt - \XXstar)\le 2r$. Recalling 
\[
T :=   \Big\lceil \frac{10}{\mu } \log \left(  10\sqrt{r}  \right)     
    \Big\rceil,
\]
we obtain
\begin{align}
    \dist \left( \XX_T, \XXstar \right) \le &  8\sqrt{r}\left(1-\frac{\mu}{10} \right)^{T}c_0 \sigma_{\min} (\XXstar) \label{eq:distxtxsta}  \\
    \overleq{(a)}
    &8\sqrt{r}c_0
    \exp \left( \frac{- T\mu  }{10}\right)
    \sigma_{\min} (\XXstar) \notag \\
    \overleq{(b)}
    & 0.1\sigma_{\min} (\XXstar). \label{ineq:proofmain2}
\end{align}
where $(a)$ uses  $ \ln (1+x) \le x $  for $x>-1$, and the inequality $(b)$ follows from the choice of  $T = \Big\lceil \frac{10}{\mu } \log \left(  10\sqrt{r}  \right)  \Big\rceil $ and the bound $c_0 \le 0.1$.   Moreover, since $\Aop$  satisfies rank-$2r$  RIP with constant $\delta_{2r}\le 0.02$ with probability at least $1-\exp(-(n_1+n_2))$ when $m\ge C(n_1+n_2)r \kappa^2$, 
Lemma \ref{lemma:Phase2} yields, for all $t\ge T$,
    \begin{align}\label{ineq:proofmain4}
     \dist \left( \XXt, \XXstar \right)
    \le 
    \left(1- 0.6\mu\right)^{t-T}
    \dist (\XX_T, \XXstar).
    \end{align}
    and
    \begin{equation} \label{eq:F_norm_dist}     
   \normf{\XX_t-\XXstar} \le 1.5  \dist\left( \XXt, \XXstar \right).
    \end{equation}
    Combining \eqref{eq:distxtxsta} with \eqref{ineq:proofmain4} gives that
\[
        \dist \left( \XXt, \XXstar \right) \le
      8\sqrt{r}\left(1-\frac{\mu}{10} \right)^{t}c_0 \sigma_{\min} (\XXstar)
\]
holds with probability at least $1- 7\exp(-\left(n_1+n_2\right))$, provided $m\ge C(n_1+n_2)r \kappa^2$. This is exactly \eqref{eq:maindist}. Finally, using \eqref{eq:F_norm_dist} yields the Frobenius error bound in Theorem~\ref{thm:main}. 
    This completes the proof.
    \end{proof}

\section{Experiment}\label{sec:experiment}
In this section, several numerical experiments are conducted to evaluate the effectiveness of ScaledGD in comparison with vanilla gradient descent (GD) \cite{charis} and Riemannian gradient descent (RGD) \cite{caifast}. The ground-truth low-rank matrix $\XXstar \in \mathbb{R}^{n_1 \times n_2}$ is constructed as follows. First, orthonormal matrices $\VXX \in \mathbb{R}^{n_1 \times r}$ and $\WXX \in \mathbb{R}^{n_2 \times r}$ are generated, where $r$ is the target rank. Next, the $r$ non-zero singular values of $\XXstar$ are drawn uniformly from $[1/\kappa,1]$ and arranged on the diagonal of $\SSigma_{\star} \in \mathbb{R}^{r \times r}$. The ground-truth matrix is then set to $\XXstar = \VXX \SSigma_{\star} \WXXT$, which has rank $r$ and condition number $\kappa$.  To ensure stable convergence for each algorithm, the step size for vanilla GD is set to $\mu=\eta / \sigma_1(\XXstar)$, where $\sigma_1(\XXstar)$ is the largest singular value of $\XXstar$; for ScaledGD and RGD,  we use $\mu = \eta$. Throughout  all experiments, we fix $\eta = 0.5$, following the step-size recommendation in  Tong et al. \cite{JMLR:v22:20-1067}.

In the first experiment, we set $n_1 = n_2 = 100$, $r = 30$, $m = 4n_1r$, and $\kappa = 5$. Figure~\ref{fig:iteration and time with re} reports the relative error versus the iteration count and versus the computational time. The results indicate that ScaledGD outperforms vanilla GD in achieving both lower relative error and  runtime, and it also slightly improves upon RGD, thereby confirming its effectiveness for low-rank matrix recovery.

\begin{figure}[htbp]
    \centering
    \includegraphics[width=0.48\textwidth]{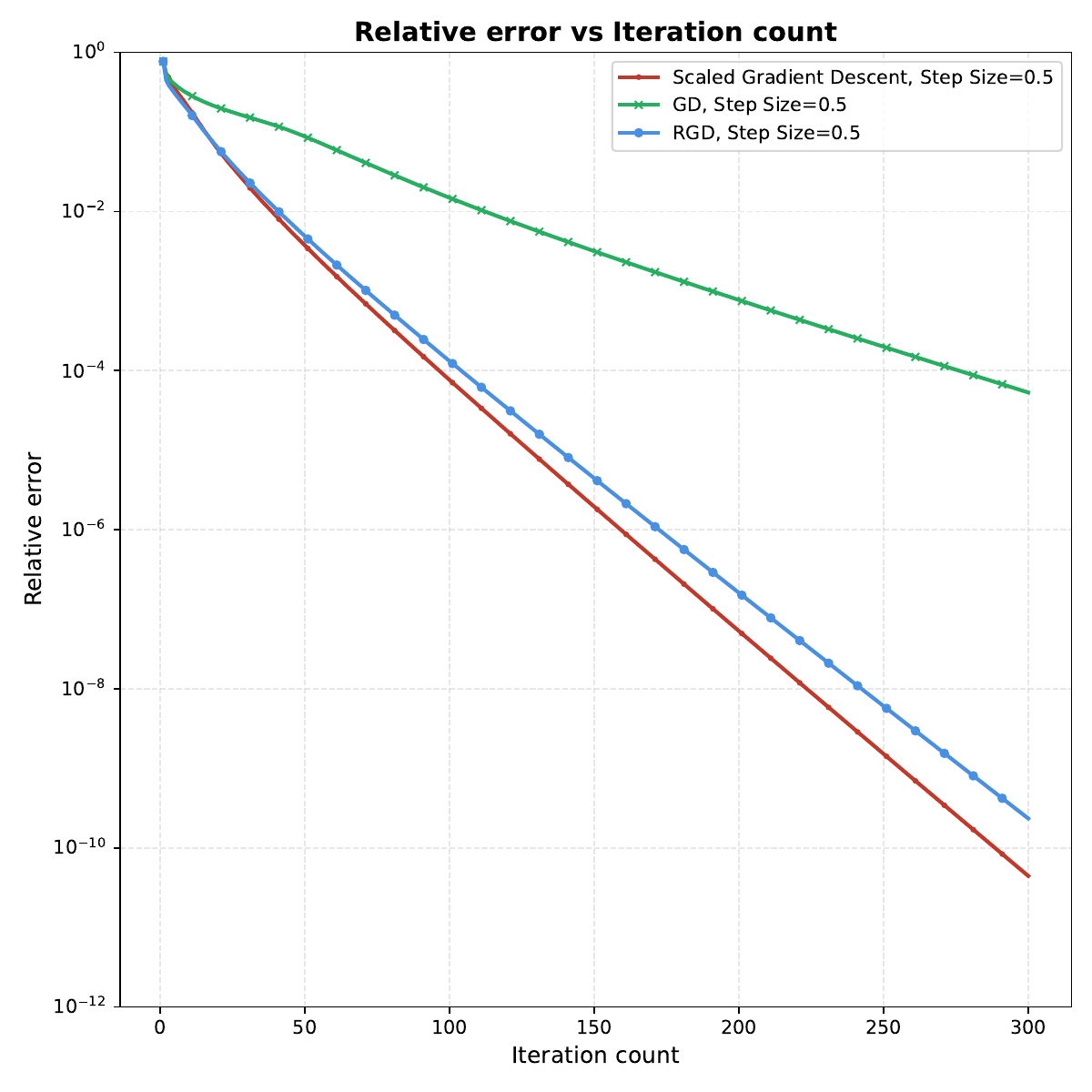}
    \hspace{0.2cm}
    \includegraphics[width=0.48\textwidth]{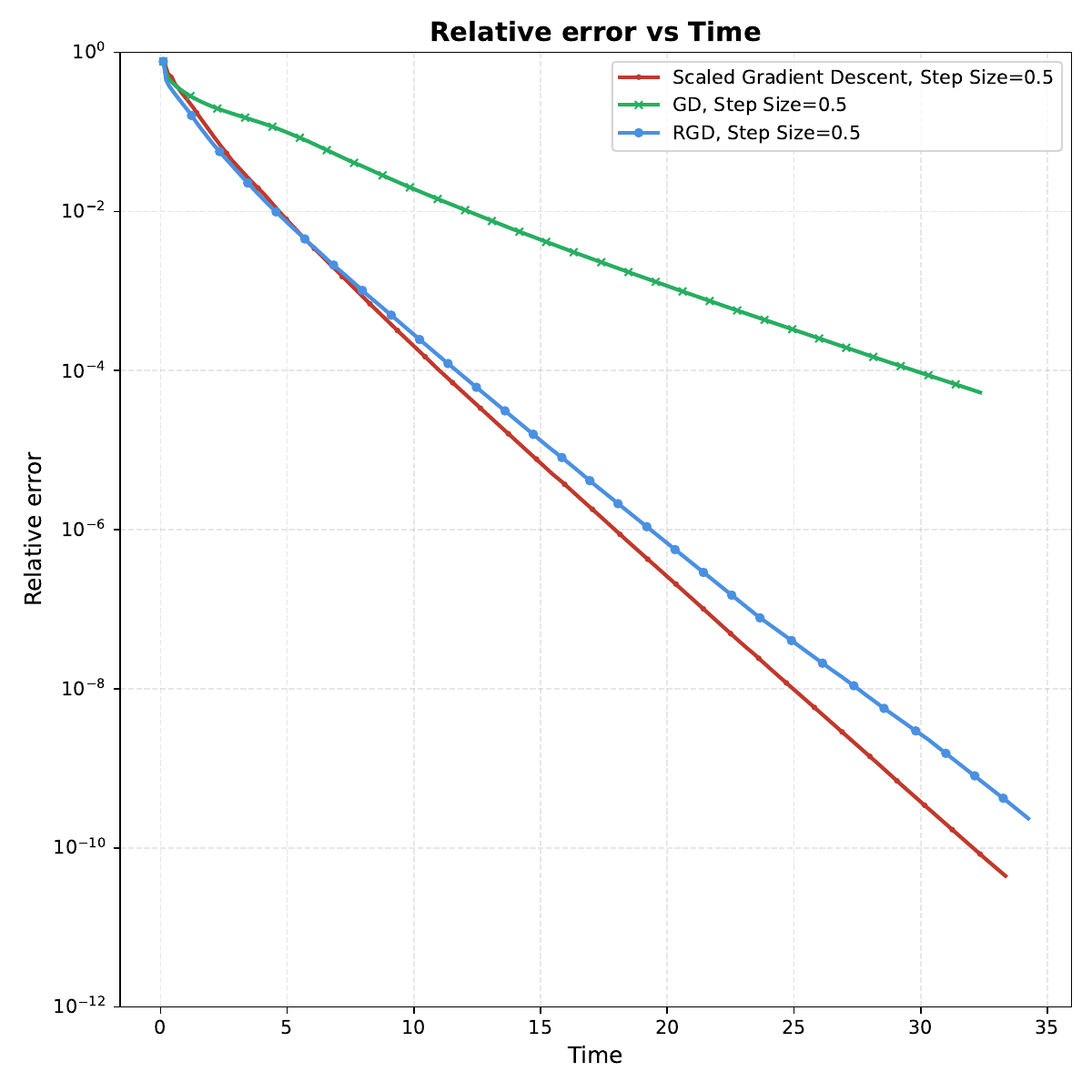}
     \caption{Relative error with iterations (left); relative error with runtime (right).}
    \label{fig:iteration and time with re}
\end{figure}

\begin{figure}[htbp]
    \centering
    \includegraphics[width=0.7\textwidth]{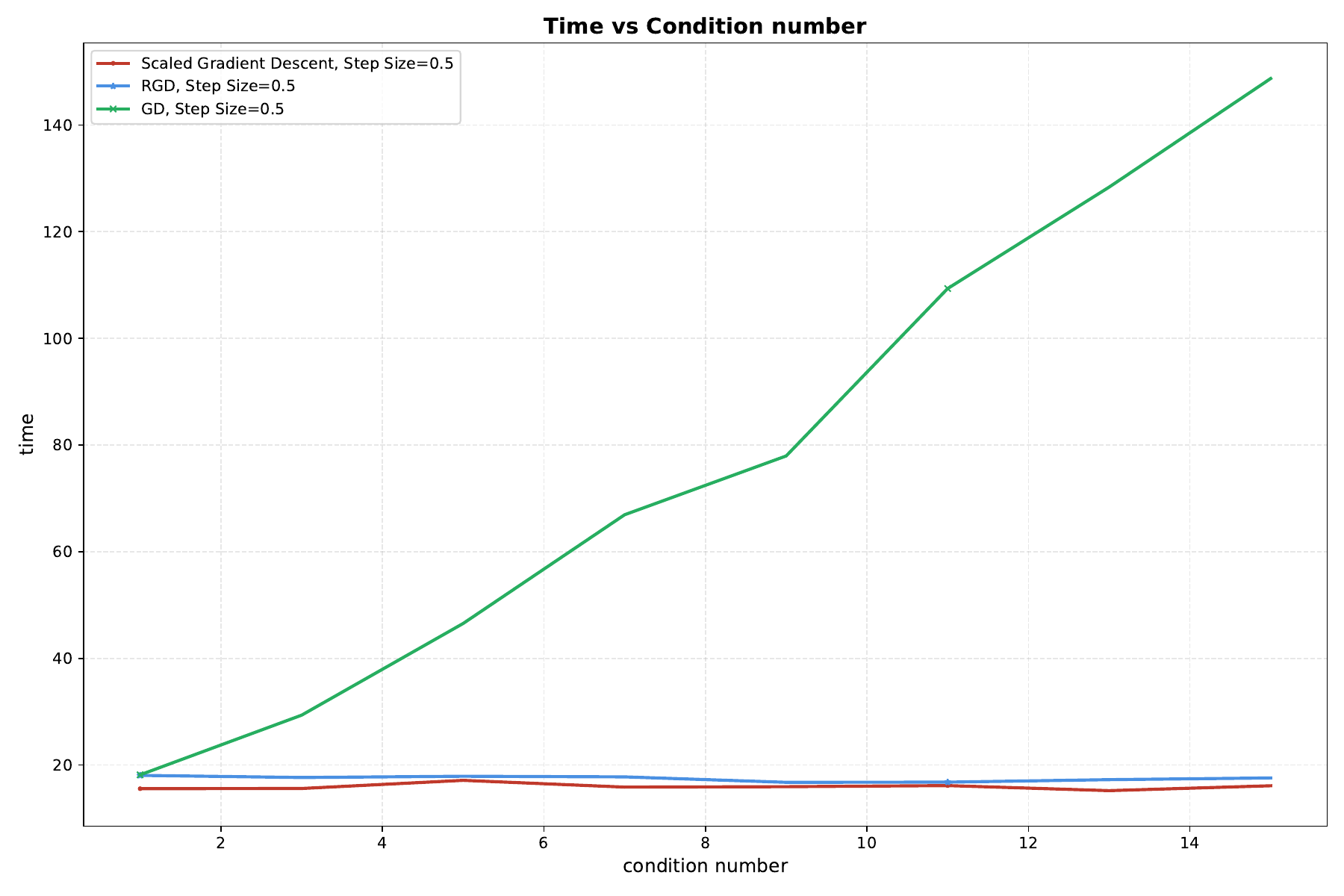}
    \caption{Time cost of different methods under varying condition numbers.}
    \label{fig:time condition}
\end{figure}
In the second experiment, the performance of ScaledGD, GD, and RGD is evaluated on ill-conditioned low-rank matrix recovery with large condition number $\kappa$. The dimensions are fixed as $n_1 = n_2 = 100$, $r = 30$, $m =5n_1r$,
and the maximum number of iterations is set to $N_1=1000$. The condition number $\kappa$ varies from $1$ to $15$, and for each method we record the time required to obtain an estimate $\XX_{t}$ satisfying  $\normf{\XX_{t}-\XXstar}   / \normf{\XXstar}\le 10^{-6}$. The results in Figure~\ref{fig:time condition} show that, as  $\kappa$ increases, the computational cost of ScaledGD and RGD remains relatively stable, whereas that of GD grows roughly linearly with $\kappa$, thereby confirming the effectiveness of ScaledGD for ill-conditioned low-rank matrix recovery.

Finally, the success rate is examined as a function of the number of measurements $m$ and the rank $r$ of the target matrix. The dimensions are fixed at $n_1 = 70$ and $n_2 = 80$, with condition number $\kappa = 5$. The rank $r$  varies from $1$ to $20$, while $m$ ranges from $1000$ to $13000$. For each $(r, m)$ pair, 10 independent trials are conducted, and a trial is declared successful if the relative error satisfies $\normf{\XX_{\NN_2} - \XXstar} / \normf{\XXstar}  \le 10^{-8}$ after $N_2=100$ iterations. As shown in Figure~\ref{fig:transition diagram}, a clear phase transition phenomenon is observed, and the phase transition boundary depends linearly on the rank $r$, which is consistent with the theoretical predictions.
\begin{figure}[htbp]
    \centering
    \includegraphics[width=0.7\textwidth]{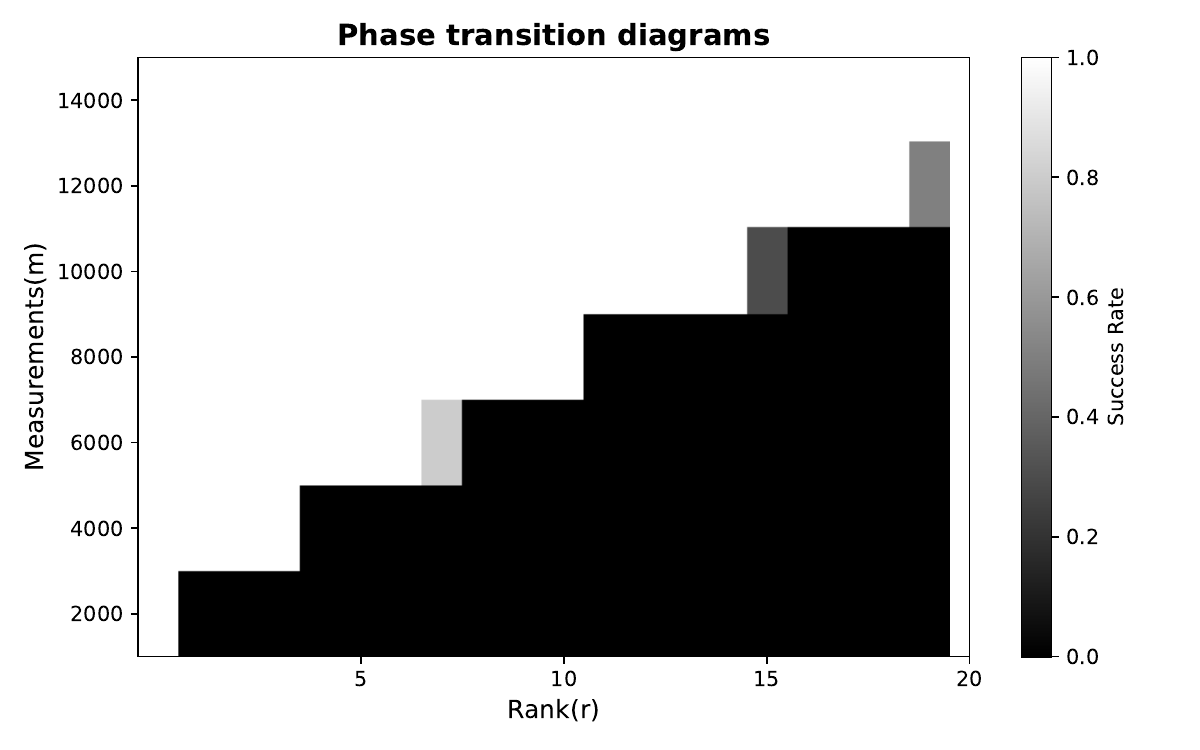}
    \caption{Phase transition diagrams: $m$ vs $r$. Black indicates failure,  and white indicates success.}
    \label{fig:transition diagram}
\end{figure}

\section{Discussions}\label{sec:discussions}
In this paper, we demonstrate that Scaled Gradient Descent can recover a rank-$r$ matrix $\XXstar \in \R^{n_1\times n_2}$  within $O\left(\log(1/\epsilon)\right)$ iterations to achieve $\epsilon$-accuracy, with an iteration complexity that is  independent of the condition number. Moreover, the sample complexity is $O (\left(n_1+n_2\right)r \kappa^2)$, which matches the optimal information-theoretic scaling. Compared with the recent work of St\"{o}ger and Zhu \cite{stoger2025nonconvexmatrixsensingbreaking}, this improves the iteration complexity from $O\left(\kappa^2 \log(1/\epsilon)\right)$ to $O\left(\log(1/\epsilon)\right)$, and, at the same time, removes the PSD assumption, extending the guarantees to general low-rank matrix recovery.

There are some interesting problems for future research:
\begin{itemize}
\item \textbf{Removing the condition number in sample complexity.} The sample complexity of ScaleGD established in this paper is $O(\left(n_1+n_2\right)r \kappa^2)$. Compared with convex methods, which achieve sample complexity $O(\left(n_1+n_2\right)r)$, there remains a gap.  In fact, this issue is shared by all existing nonconvex methods for low-rank matrix recovery, and removing the dependence on the condition number in the sample complexity is still an open problem. The requirement of $O (r\left(n_1+n_2\right)\kappa^2)$ measurements arises solely from the spectral initialization step. Notably, the Stage-Alternating Minimization algorithm \cite{jain2013low} succeeds in eliminating the condition number dependence in its sampling requirement. This suggests that incorporating similar ideas into ScaledGD may offer a promising direction for further reducing its sample complexity.
\item \textbf{Convergence under random and small initialization.} The convergence analysis of ScaledGD in this paper relies critically on spectral initialization, whereas in practice random, small-norm initializations are often preferred. Establishing rigorous convergence guarantees for ScaledGD under such random and small initializations is an interesting and important direction for future work.
\item \textbf{Overparameterization in low-rank matrix sensing.} In many practical scenarios, the true rank of the target matrix is unknown, leading naturally to overparameterized models. Investigating whether the techniques developed here can be extended to handle overparameterized settings, such as those encountered in PrecGD and ScaledGD($\lambda$) \cite{xu2023power,zhang2021preconditioned}, constitutes another promising direction for future research.
\end{itemize}


\appendix

\section{Proofs of supporting lemmas in Section \ref{sec:proof main result}}
In this section, we prove Lemmas \ref{lemma:auxsequencecloseness}, \ref{lemma:localconv}, and \ref{Le:4_6}.

\subsection{Proof of Lemma \ref{lemma:auxsequencecloseness}}\label{subsec:proofauxsequencecloseness}
To prove Lemma \ref{lemma:auxsequencecloseness}, we first need the follow Lemma, which bounds $ \normf{ \XXt -\XXtw } $ in terms  of $\normf{\VXXT \left( \XXt -\XXtw \right) } $ and $\normf{\left( \XXt -\XXtw \right)\WXX} $.

\begin{lemma}\label{lemma:auxcloseness1}
Let $\left(\ww,\vv\right) \in \mathcal{N}$, where $\mathcal{N}$ is given in \eqref{eq:fixedNep}. Assume that   
\begin{align}
    \max\{\specnorm{\VXXPT \VLLt}, \quad \specnorm{\WXXPT \VRRt} \}&\le \frac{1}{8},\label{assump:lemma12_1} \\
    \specnorm{ \XXt - \XXstar}
    &\le
    \frac{ \sigma_{\min} \left( \XXstar \right)}{80},\label{assump:lemma12_2}\\
    \normf{ \XXt - \XXtw }
    &\le 
    \frac{\sigma_{\min} \left( \XXstar \right)}
    {80  }.\label{assump:lemma12_3}
\end{align}
Then it holds
\begin{align}\label{ineq:lemmaequ2}
    \normf{\XXt - \XXtw}
    \le
    \frac{5}{4}\left(
        \normf{ \VXXT \left(  \XXt - \XXtw \right) }+\normf{ \left( \XXt - \XXtw \right) \WXX}\right)
\end{align}
and 
\begin{align}\label{ineq:lemmaequ1}
    \normf{ \VXXPT \left(\XXt - \XXtw \right) \WXXP} 
    \le
    \frac{1}{4}\left(
        \normf{ \VXXT \left(  \XXt - \XXtw \right) }+\normf{ \left( \XXt - \XXtw \right) \WXX}\right).
\end{align}
\end{lemma}

\begin{proof}
Let  $\XXt=\VLLt\SSigma_{t}\VRRt^\top $ be the compact SVD of $\XXt$. Since $\VLLt\VLLtP$ and $\VLLtw\VLLtwP$ are the orthogonal projection matrices onto the column spaces of $\XXt$ and $\XXtw$, respectively,  we have
    \begin{align} 
        &\normf{\VXXPT \left(\XXt - \XXtw \right) \WXXP}=\normf{\VXXPT \left(\VLLt\VLLtP\XXt - \VLLtw\VLLtwP\XXtw \right) \WXXP} \notag \\
        &\overleq{(\textup{i})} 
        \normf{\VXXPT \VLLt\VLLtP\left(\XXt -\XXtw\right)\WXXP } \notag \\&+\normf{\VXXPT \left(\VLLt\VLLtP- \VLLtw\VLLtwP\right)\XXtw\WXXP } \notag 
 \end{align}
   \begin{align} 
    &\overleq{(\textup{ii})} 
         \specnorm{\VXXPT \VLLt}\normf{\XXt -\XXtw }+\normf{\VLLt\VLLtP- \VLLtw\VLLtwP}\specnorm{\left(\XXtw-\XXstar\right)\WXXP} \notag \\
               &\overleq{(\textup{iii})}  \frac{1}{8}\normf{\XXt -\XXtw }+\frac{\sqrt{2}\normf{\XXt-\XXtw}}{\sigma_{\min} (\XXt)}\specnorm{\left(\XXtw-\XXstar\right)\WXXP}\notag \\
        &\overleq{(\textup{iv})}\frac{1}{5}\normf{\XXt -\XXtw}, \label{ineq:lemmaequ4} 
        \end{align}
where  $(\textup{i})$ uses the triangle inequality,  $(\textup{ii})$ uses $\XXstar\WXXP=0$, and  $(\textup{iii})$ follows from \eqref{assump:lemma12_1} and Lemma \ref{lemma:projection matrices}. For $(\textup{iv})$, we use the fact that  $\sigma_{\min} (\XXt)\ge \sigma_{\min} (\XXstar)/{2}$ 
     and 
    \begin{align*}
        \specnorm{\left(\XXtw-\XXstar\right)\WXXP}&\le \specnorm{\XXt-\XXstar}+\specnorm{\XXtw-\XXt}\\
        &\le \frac{{\sigma_{\min}\left( \XXstar \right)}}{40},
    \end{align*}
    due to the assumptions \eqref{assump:lemma12_2} and \eqref{assump:lemma12_3}.  Next, we bound $\normf{\XXt -\XXtw}$. By triangle inequality,
    \begin{align*}
        &\normf{\XXt - \XXtw}\\
        &\le
        \normf{\VXXT \left( \XXt - \XXtw \right)}
        +
        \normf{ \VXXPT \left( \XXt - \XXtw \right) \WXX}+
        \normf{\VXXPT \left(\XXt - \XXtw \right) \WXXP}\\
        &\le
        \normf{\VXXT \left( \XXt - \XXtw \right)}+\normf{ \left( \XXt - \XXtw \right) \WXX}+\frac{1}{5} \normf{  \XXt - \XXtw},
    \end{align*}
    where the last inequality follows from inequality \eqref{ineq:lemmaequ4}. Rearranging gives
    \begin{align*}
        \normf{\XXt - \XXtw}
        &\le 
        \frac{1}{1-\frac{1}{5}}\left(
        \normf{ \VXXT \left(  \XXt - \XXtw \right) }+\normf{ \left( \XXt - \XXtw \right) \WXX}\right)\\
        &\le \frac{5}{4}\left(
        \normf{ \VXXT \left(  \XXt - \XXtw \right) }+\normf{ \left( \XXt - \XXtw \right) \WXX}\right),
    \end{align*}
    which is \eqref{ineq:lemmaequ2}. Substituting this bound into \eqref{ineq:lemmaequ4} yields \eqref{ineq:lemmaequ1}. This completes the proof.
\end{proof}

Now, we are ready to prove Lemma \ref{lemma:auxsequencecloseness}.

\begin{proof}[Proof of Lemma \ref{lemma:auxsequencecloseness}]
        
\noindent For convenience, denote
    \begin{align*}
     \MM_t     :&= \left( \Aops \right) \left( \XXstar - \XXt \right)=
            \XXstar - \XXt 
            + \bracing{=: \EEb_t}{\left( \Aops - \IdOp \right) \left(\XXstar - \XXt\right)}, \\
     \MM_{t}^{\left(\ww,\vv \right)} &:= \left( \Aopws \right) \left( \XXstar - \XXtw \right)\\
           &=
            \XXstar - \XXtw
            + \bracing{=: \EEb_{t}^{\left(\ww,\vv \right)}}{\left( \Aopws - \IdOp \right) \left(\XXstar - \XXtw\right)}.
    \end{align*} 
    From the update rule \eqref{equ:gradientdescentdefinition} and the corresponding virtual iteration in Algorithm~\ref{algorithm:SGDvirtual},
    \[
     \LLtplus=  \LLt+ \mu \MM_t  \RRt\left(\RRtT\RRt\right)^{-1},  \quad \RRtplus= 
        \RRt+ \mu \MM_t^\top  \LLt\left(\LLtT\LLt\right)^{-1},
    \]
    and
    \begin{align*}
     \LLtplusw&= \LLtw+ \mu \MM_{t}^{\left(\ww,\vv \right)}  \RRtw\left(\RRtwT\RRtw\right)^{-1}, \\
      \quad \RRtplusw &= \RRtw+ \mu {\MM_{t}^{\left(\ww,\vv \right)}}^\top \LLtw\left(\LLtwT\LLtw\right)^{-1}.
    \end{align*}
   Let  $\VLLt\SSigma_{t}\VRRt^\top $ be the  compact SVD  of $\XXt=\LLt \RRtT$,  with $\VLLtP\VLLt=\II_{r}$ and  $\VRRt^\top \VRRt=\II_{r}$. There exists an invertible matrix $\QQ$
     such that $\LLt=\VLLt\SSigma_{t}^{1/2}\QQ$ and $\RRt=\VRRt\SSigma_{t}^{1/2}\QQ^{-\top}$. It follows that 
  \begin{equation*}
     \LLt \left(\LLtT\LLt\right)^{-1}\LLtT=\VLLt\VLLtP, \qquad  \RRt \left(\RRtT\RRt\right)^{-1}\RRtT=\VRRt\VRRtP;
  \end{equation*}
and
\[
     \RRt\left(\RRtT\RRt\right)^{-1}\left(\LLtT\LLt\right)^{-1}\LLtT=
     \VRRt \SSigma_{t}^{-1} \VLLtP.
\]
A direct calculation gives
    \begin{align}
        \XXtplus
        &=
       \XXt+\mu\MM_t\VRRt\VRRtP
        +\mu\VLLt\VLLtP\MM_t
        +\mu^{2}\MM_t\RRt\left(\RRtT\RRt\right)^{-1}\left(\LLtT\LLt\right)^{-1}\LLtT\MM_t \notag \\
        &=
        \XXt+\mu\left( \XXstar -\XXt\right)\VRRt\VRRtP
        +\mu\VLLt\VLLtP\left( \XXstar -\XXt\right)
        +\mu\EEb_t\VRRt\VRRtP+\mu \VLLt\VLLtP\EEb_t  \notag\\
        &+\mu^{2}\MM_t\VRRt \SSigma_{t}^{-1} \VLLtP\MM_t.  \label{eq:Xt1}
    \end{align}
    Similarly, let  $\XXtw=\VLLtw \SSigma_{t}^{\left(\ww,\vv \right)} \VRRtwP $ be the compact SVD. Then
    \begin{align}
       \XXtwplus&=
        \XXtw+\mu\left( \XXstar -\XXtw\right)\VRRtw\VRRtwP
        +\mu\VLLtw\VLLtwP\left( \XXstar -\XXtw\right)  \notag \\
        &+\mu\EEb_{t}^{\left(\ww,\vv \right)}\VRRtw\VRRtwP+\mu \VLLtw\VLLtwP\EEb_{t}^{\left(\ww,\vv \right)} \notag \\
        &+\mu^{2}\MM_{t}^{\left(\ww,\vv \right)}\VRRtw\left({\SSigma_{t}^{\left(\ww,\vv \right)}}\right) ^{-1}\VLLtwP\MM_{t}^{\left(\ww,\vv \right)}.     \label{eq:Xt1uv}
    \end{align}
    Subtracting \eqref{eq:Xt1uv} from \eqref{eq:Xt1} yields
    \begin{align}
        &\XXtplus-\XXtwplus  \label{ineq:intern42}\\
        &=  \underbrace{\left( \II_{n_1}  -\mu \VLLt\VLLtP  \right) \left(\XXt-\XXtw \right)\left(\II_{n_2}  -\mu \VRRt\VRRtP  \right)}_{\MM_1}  \notag \\
        &-\mu^2  \underbrace{\VLLt\VLLtP\left(\XXt-\XXtw \right)\VRRt\VRRtP}_{\MM_2}+\mu \underbrace{ \left(\VLLt\VLLtP-\VLLtw\VLLtwP \right)\left( \XXstar -\XXtw\right)}_{\MM_3} \notag\\
        &+\mu \underbrace{\left( \XXstar -\XXtw\right)\left(\VRRt\VRRtP-\VRRtw\VRRtwP \right)}_{\MM_4} \notag \\
         &+\mu  \underbrace{\xkh{ \EEb_t\VRRt\VRRtP - \EEb_{t}^{\left(\ww,\vv \right)} \VRRtw\VRRtwP+\VLLt\VLLtP\EEb_t
        -  \VLLtw\VLLtwP\EEb_{t}^{\left(\ww,\vv \right)} }}_{\MM_5} \notag \\
         &+\mu^2  \underbrace{ \xkh{\MM_t \VRRt \SSigma_{t}^{-1} \VLLtP\MM_t-\MM_{t}^{\left(\ww,\vv \right)} \VRRtw \left(\SSigma_{t}^{\left(\ww,\vv \right)} \right)^{-1} \VLLtwP\MM_{t}^{\left(\ww,\vv \right)}}}_{\MM_6}.  \notag 
    \end{align}
    To prove the lemma, upper bounds are derived for  $ \normf{ \MM_i} $, where $i=1,2,\ldots, 6$.\\
        
    \noindent\textbf{Estimating $ \normf{ \VXXT\MM_1}$ and $ \normf{\MM_1\WXX}$:} Observe that
    \begin{align*}
        \VXXT\MM_1&=\left( \VXXT  -\mu \VXXT \VLLt\VLLtP  \right)\left(\VXX \VXXT +\VXXP\VXXPT\right) \left(\XXt-\XXtw \right)\left(\II_{n_2}  -\mu \VRRt\VRRtP  \right)\\
        &=\left(\II_{n_1}  -\mu \VXXT \VLLt\VLLtP\VXX\right)\VXXT\left(\XXt-\XXtw\right)\left(\II_{n_2}  -\mu \VRRt\VRRtP  \right)\\
        &-\mu \VXXT\VLLt\VLLtP \VXXP\VXXPT\left(\XXt-\XXtw\right)\left(\II_{n_2}  -\mu \VRRt\VRRtP  \right).
    \end{align*}
    By triangle inequality and note that $\specnorm{\II_{n_2}  -\mu \VRRt\VRRtP } \le 1$, we have 
    \begin{align*}       
       \normf{\VXXT\MM_1}
        &\le 
        \bracing{=:L_1}{\specnorm{\II_{n_1}  -\mu \VXXT \VLLt\VLLtP\VXX } \normf{\VXXT\left(\XXt-\XXtw\right)}}\\
        &+
        \bracing{=:L_2}{\mu \normf{\VXXT\VLLt\VLLtP \VXXP\VXXPT\left(\XXt-\XXtw\right)}}.
    \end{align*}
    For $L_1$, it is easy to see
    \begin{align*}
        L_1 =
        &\left(1-\mu \sigma^2_{\min}(\VXXT\VLLt)\right)\normf{\VXXT\left(\XXt-\XXtw\right)}\\
        &\le  \left(1-\left(1-\constone^2\right)\mu \right)\normf{\VXXT\left(\XXt-\XXtw\right)},
    \end{align*}
    where the  inequality comes from 
        \begin{align*}
       \sigma^2_{\min}(\VXXT\VLLt) \ge 1-\constone^2,
    \end{align*}
    which can be verified from the fact that   
    $\VLLtP\VLLt=\left(\VXXT\VLLt\right)^\top\VXXT\VLLt+\left(\VXXPT\VLLt\right)^\top\VXXPT\VLLt$ and assumption \eqref{assump:closeness5}.  For  $L_2$, we have 
     \begin{align*}
         \LL_2 &=\normf{\VXXT\VLLt\VLLtP \VXXP\VXXPT\left(\XXt-\XXtw\right)\left( \WXX \WXXT +\WXXP\WXXPT\right)}\\
         &\le
         \specnorm{\VXXT\VLLt\VLLtP \VXXP}\normf{\VXXPT\left(\XXt-\XXtw\right)\WXX\WXXT}\\
         &+\specnorm{\VXXT\VLLt\VLLtP\VXXP}\normf{\VXXPT\left(\XXt -\XXtw\right)\WXXP\WXXPT}\\
         &\overleq{(a)}  \constone\left(
         \normf{\left(\XXt-\XXtw\right)\WXX}+\normf{\VXXPT\left(\XXt-\XXtw\right)\WXXP}\right)\\
         &\overleq{(b)} 
         \frac{\constone}{4}
         \normf{\VXXT\left(\XXt-\XXtw\right)}+\frac{5\constone}{4}\normf{\left(\XXt-\XXtw\right)\WXX}, 
     \end{align*}
     where the inequality $(a)$ follows from assumption\eqref{assump:closeness5} and the inequality $(b)$ comes from Lemma \ref{lemma:auxcloseness1}. Putting  $L_1$ and $L_2$ together, we obtain  
     \begin{align*}
         &\normf{\VXXT\MM_1}\\
         &\le
          \left(1-\left(1-\constone^2-\frac{\constone}{4}\right)\mu \right)\normf{\VXXT\left(\XXt-\XXtw\right)} +\frac{5\constone\mu}{4}\normf{\left(\XXt-\XXtw\right)\WXX}\\   
         &\le 
        \left(1-\frac{15\mu}{16}\right)\normf{\VXXT\left(\XXt-\XXtw\right)}+\frac{5\constone\mu}{4}\normf{\left(\XXt-\XXtw\right)\WXX},
     \end{align*}
    where the last inequality follows from the assumption that $c_1\le 1/20$.
    Similarly, we can obtain that
   \[
        \normf{\MM_1\WXX}\le
        \left(1-\frac{15\mu}{16}\right)\normf{\left(\XXt-\XXtw\right)\WXX}+\frac{5\constone\mu}{4}\normf{\VXXT\left(\XXt-\XXtw\right)}.
\]

    \noindent\textbf{Estimating $ \normf{\MM_2} $:} Triangle inequality gives
  \[
       \normf{\MM_2} \le
       \specnorm{\VLLt\VLLtP}\normf{\XXt-\XXtw }\specnorm{\VRRt\VRRtP} \le 
     \normf{\XXt-\XXtw}.
\]

    \noindent\textbf{Estimating $ \normf{\MM_3} $:} Note that 
\[
        \MM_3 = \left(\VLLt\VLLtP-\VLLtw\VLLtwP \right) \xkh{ \left( \XXstar -\XXt\right)+\left(\XXt-\XXtw\right)}.
\]
    Applying the triangle inequality, we obtain
    \begin{align*}
       \normf{\MM_3} 
       &\le \normf{\VLLt\VLLtP-\VLLtw\VLLtwP} \xkh{ \specnorm{\XXstar -\XXt}+\specnorm{\XXt-\XXtw}} \\
       &\overleq{(\textup{i})}  \frac{\sqrt{2}\normf{\XXt-\XXtw}}{\sigma_{\min} (\XXt)} \cdot \left(\consttwo\sigma_{\min}(\XXstar)+\specnorm{\XXt-\XXtw}\right) \\
       &\overleq{(\textup{ii})} \frac{\sqrt{2}\left(\consttwo+\constthree\right)}{1-\consttwo}\normf{\XXt-\XXtw} , 
    \end{align*}
where the inequality $(\textup{i})$ arises from Lemma \ref{lemma:projection matrices} and assumption\eqref{assump:closeness7}, and the inequality $(\textup{ii})$ follows from Weyl’s inequality, 
$\sigma_{\min} (\XXt) \ge \left(1-\consttwo\right)\sigma_{\min} (\XXstar)$ and assumption \eqref{assump:closeness8}.

    \noindent\textbf{Estimating $ \normf{\MM_4} $:} Using the similar arguments to $\MM_3$, we have 
    \begin{align*}
       \normf{\MM_4} 
       \le  \frac{\sqrt{2}\left(\consttwo+\constthree\right)}{1-\consttwo}\normf{\XXt-\XXtw}
       .
    \end{align*}

    \noindent\textbf{Estimating $ \normf{\MM_5} $:} By Lemma \ref{lemma:key}, we have
    \begin{align} \label{eq:constfive}
    \specnorm{ \left( \Aops -\IdOp \right) \left( \XXstar - \XXt \right)  }
        &\le  4c^{\prime} \specnorm{ \XXstar - \XXt }+ 6c^{\prime}\supw  \normf{\XXt-\XXtw}  \\
        &\overleq{(a)} 
        \left(4c^{\prime}\consttwo+6c^{\prime}\constthree\right)\sigma_{\min} (\XXstar) \notag 
        =:\constfive \sigma_{\min} (\XXstar),
    \end{align}
where $(a)$ uses assumption\eqref{assump:closeness7} and \eqref{assump:closeness8}. Next, we decompose $ \normf{\MM_5} $ as
    \begin{equation}
       \MM_5
        =
        \bracing{=:\OO_1}{\VLLt\VLLtP\EEb_t- \VLLtw\VLLtwP\EEb_{t}^{\left(\ww,\vv \right)}}+
        \bracing{=:\OO_2}{\EEb_t\VRRt\VRRtP-\EEb_{t}^{\left(\ww,\vv \right)}\VRRtw\VRRtwP}. \label{equ:intern1}
    \end{equation}
    For the first term, observe that
    \begin{align*}
        &\normf{\OO_1} \\& \overleq{(a)}
         \normf{\VLLt\VLLtP\EEb_t- \VLLtw\VLLtwP\EEb_t}+\normf{\VLLtw\VLLtwP\EEb_t- \VLLtw\VLLtwP \EEb_{t}^{\left(\ww,\vv \right)}}\\
       & \le
         \specnorm{\EEb_t}\normf{\VLLt\VLLtP-\VLLtw\VLLtwP}+\normf{\VLLtw\VLLtwP\left(\EEb_t-\EEb_{t}^{\left(\ww,\vv \right)}\right)}\\
        &\overleq{(b)}
         \specnorm{\EEb_t}\frac{\sqrt{2}}{\left(1-\consttwo\right)\sigma_{\min} (\XXstar)}\normf{\XXt-\XXtw}+2c^{\prime} \specnorm{\XXstar - \XXt}
       +  4c^{\prime} \normf{\XXt - \XXtw}\\
      &\overleq{(c)}
      \left(4c^{\prime}+\frac{\sqrt{2}\constfive}{1-\consttwo} \right) \normf{\XXt - \XXtw}+2c^{\prime}\specnorm{\XXstar - \XXt},
    \end{align*} 
    where $(a)$ is by the triangle inequality,  $(b)$  uses Lemma \ref{lemma:projection matrices} together with
    \begin{align*}
    \EEb_t-\EEb_{t}^{\left(\ww,\vv \right)}
    =&\left( \Aops - \IdOp \right)\left(\XXstar-\XXt\right)
        -\left( \Aopws - \IdOp \right)\left(\XXstar-\XXtw\right)\\
    =&\left( \Aops - \Aopws \right) \left( \XXstar - \XXt \right)+\left(\IdOp-\Aopws\right)\left(\XXt-\XXtw\right), 
    \end{align*}
    and 
    \begin{align*}
    \normf{\VLLtwP\left(\EEb_t-\EEb_{t}^{\left(\ww,\vv \right)}\right)}\le 
  2c^{\prime}\specnorm{\XXstar - \XXt}
     +4c^{\prime} \normf{ \XXt - \XXtw },
    \end{align*}
    due to Lemma \ref{lemma:auxestimates}.
   The  inequality $(c)$ follows from assumption \eqref{assump:closeness7} and inequality \eqref{eq:constfive}.  The same argument applies to  $\OO_2$, so its bound is omitted. Consequently,
    \begin{align*}
        \normf{\MM_5}
        \le 
        \left(8 c^{\prime} +\frac{2\sqrt{2}\constfive}{1-\consttwo} \right) \normf{\XXt - \XXtw }+4c^{\prime}\specnorm{\XXstar - \XXt}.
    \end{align*}
    
    \noindent\textbf{Estimating $\normf{\MM_6 }$:}   To handle $\MM_6$, we first decompose it as 
    \begin{align*}
        \MM_6         &=
        \bracing{=:\LL_1}{\MM_t\left(\VRRt\SSigma_{t}^{-1}\VLLtP-\VRRtw\left({\SSigma_{t}^{\left(\ww,\vv \right)}}\right)^{-1}\VLLtwP\right)\MM_t} \\
        &+\bracing{=:\LL_2}{\left(\MM_t-\MM_{t}^{\left(\ww,\vv \right)}\right)\VRRtw\left({\SSigma_{t}^{\left(\ww,\vv \right)}}\right)^{-1}\VLLtwP\MM_t}\\
        &+\bracing{=:\LL_3}{\MM_{t}^{\left(\ww,\vv \right)}\VRRtw\left({\SSigma_{t}^{\left(\ww,\vv \right)}}\right)^{-1}\VLLtwP\left(\MM_t-\MM_{t}^{\left(\ww,\vv \right)}\right)}.
    \end{align*}
    We now bound the Frobenius norm of each term.  Recall that  the compact SVD of $\XXt$  is $\XXt=\VLLt\SSigma_{t}\VRRt^\top$, so its pseudo-inverse is $\VRRt\SSigma_{t}^{-1}\VLLtP$. Similarly,  the pseudo-inverse of $\XXtw$ is $\VRRtw\left({\SSigma_{t}^{\left(\ww,\vv \right)}}\right)^{-1}\VLLtwP$.  Since both $\XXt$ and $\XXtw$ have rank $r$, Lemma \ref{lemma:fake inverse} gives
    \begin{align}
        &\normf{\VRRt\SSigma_{t}^{-1}\VLLtP-\VRRtw\left({\SSigma_{t}^{\left(\ww,\vv \right)}}\right)^{-1}\VLLtwP}\notag \\
        &\le  
        3\specnorm{\VRRt\SSigma_{t}^{-1}\VLLtP}\specnorm{\VRRtw\left({\SSigma_{t}^{\left(\ww,\vv \right)}}\right)^{-1}\VLLtwP}\normf{\XXt-\XXtw} \notag \\
       & \le  3\sigma_{\min} (\XXt)\sigma_{\min} (\XXtw)\normf{\XXt-\XXtw} \notag 
        \\
        &\le  
        \frac{24}{\sigma^2_{\min} (\XXstar)}\normf{\XXt-\XXtw},  \label{eq:fake inverse estimate}
    \end{align}
     where the last inequality arises from that  $\sigma_{\min} (\XXtw) \ge \sigma_{\min} (\XXstar)/{4}$ and $\sigma_{\min} (\XXt) \ge \sigma_{\min} (\XXstar)/{2}$. Moreover, 
    \begin{align}
     \specnorm{\MM_t} \le
     &\specnorm{\XXstar - \XXt}
     + \specnorm{ \left(\Aops - \IdOp \right) (\XXstar - \XXt)}\nonumber \\
     \le
     & \left(\consttwo+\constfive\right) \sigma_{\min} \left( \XXstar \right),
     \label{ineq:weakboundintern9}
    \end{align}
where the last inequality comes from assumption \eqref{assump:closeness7} and inequality \eqref{eq:constfive}.  Combining \eqref{eq:fake inverse estimate} and \eqref{ineq:weakboundintern9} yields
    \begin{align*}
    \normf{\LL_1}
    &\le 
    \specnorm{\MM_t}^2\normf{\VRRt\SSigma_{t}^{-1}\VLLtP-\VRRtw\left({\SSigma_{t}^{\left(\ww,\vv \right)}}\right)^{-1}\VLLtwP}
    \\
    &\le 
    24\left(\consttwo+\constfive \right)^2\normf{\XXt-\XXtw}.
    \end{align*}
    For $\normf{\LL_2}$ and $\normf{\LL_3}$, note that 
   \begin{align*}
     \MM_t - \MM_{t}^{\left(\ww,\vv \right)}
     =
     \left( \Aops - \Aopws \right) \left( \XXstar - \XXt \right)
     -
     \left( \Aopws \right) \left(\XXt - \XXtw \right) .
   \end{align*}
Therefore,
   \begin{align}
     &\normf{ \left( \MM_t - \MM_{t}^{\left(\ww,\vv \right)} \right) \WW_{t}^{\left(\ww,\vv \right)} } \notag \\
     &\le
     \normf{ \left[\left( \Aops - \Aopws \right) \left( \XXstar - \XXt \right)\right] \WW_{t}^{\left(\ww,\vv \right)}} 
     + \normf{\XXt - \XXtw}
      \nonumber \\
     &+\normf{\left[ \left( \Aopws -\IdOp \right) \left(\XXt - \XXtw \right) \right] \WW_{t}^{\left(\ww,\vv \right)}}
     \nonumber \\
     &\le 2c^{\prime}
     \specnorm{\XXstar - \XXt}
     +
     \left( 4c^{\prime}+1 \right) \normf{ \XXt - \XXtw }, 
     \label{ineq:weakboundintern7}
   \end{align}
   where the last inequality follows from Lemma~\ref{lemma:auxestimates}. Similarly,
   \begin{align}
     \normf{ \VLLtwP \left( \MM_t - \MM_{t}^{\left(\ww,\vv \right)} \right) }
     \le 2c^{\prime}
     \specnorm{\XXstar - \XXt}
     +  \left( 4c^{\prime}+1 \right) \normf{ \XXt - \XXtw }.
      \label{ineq:weakboundintern7_2}
    \end{align}
Furthermore,
   \begin{align}
     \specnorm{\MM_{t}^{\left(\ww,\vv \right)} \WW_{t}^{\left(\ww,\vv \right)}}
     &\le 
     \specnorm{\MM_t}
     +
     \normf{\left(\MM_t -\MM_{t}^{\left(\ww,\vv \right)} \right) \WW_{t}^{\left(\ww,\vv \right)}} \nonumber \\
     &\overleq{(a)} 
     \left(\consttwo+\constfive\right) \sigma_{\min} (\XXstar)
     +
     2 c^{\prime} \specnorm{\XXstar - \XXt} 
     +\left( 4 c^{\prime}+1 \right) \normf{\XXt - \XXtw} \nonumber \\
     &\overleq{(b)} 
    \sigma_{\min} (\XXstar), \label{ineq:weakboundintern8}
   \end{align}
   where the inequality $(a)$ follows from \eqref{ineq:weakboundintern9} and \eqref{ineq:weakboundintern7}, and the
   inequality $(b)$ is a consequence of  assumptions \eqref{assump:closeness7} and \eqref{assump:closeness8} with $\consttwo,c^{\prime}  \le 0.1$. With those in place, one has
     \begin{align*}
         \normf{\LL_2}
        &\le
        \normf{\left(\MM_t-\MM_{t}^{\left(\ww,\vv \right)}\right)\VRRtw}\specnorm{\MM_t}\frac{1}{\sigma_{\min} (\XXtw)}\\
        &\le
        4\left(\consttwo+\constfive \right)\left[2c^{\prime}\specnorm{\XXstar - \XXt}
     +  \left(4c^{\prime} +1\right) \normf{\XXt - \XXtw}\right],
     \end{align*}
     where we use \eqref{ineq:weakboundintern7} and $\sigma_{\min} (\XXtw) \ge \sigma_{\min} (\XXstar)/{4}$ in the last inequality. Similarly,
    \begin{align*}
        \normf{\LL_3}
        \le
        &\normf{\VLLtwP\left(\MM_t-\MM_{t}^{\left(\ww,\vv \right)}\right)}\specnorm{\MM_{t}^{\left(\ww,\vv \right)}\VRRtw}\frac{1}{\sigma_{\min} (\XXtw)}\\
        \le
        &4\left[2c^{\prime} \specnorm{\XXstar - \XXt}
     +  \left(4c^{\prime} +1\right) \normf{\XXt - \XXtw}\right],
    \end{align*}
   using \eqref{ineq:weakboundintern7_2} and \eqref{ineq:weakboundintern8}. Summing the three bounds,
    \begin{align}
        \normf{\MM_6}
        &\le 
        4\left[\left(1+\consttwo+\constfive \right)\left(4c^{\prime} +1\right)+6\left(\consttwo+\constfive\right)^2\right] \normf{\XXt - \XXtw} \notag\\
        &+8\left(\consttwo+\constfive+1\right)c^{\prime}\specnorm{\XXstar-\XXt}  \notag \\
         &\le 5\normf{\XXt- \XXtw}+\specnorm{\XXstar-\XXt}. \label{eq:M6}
    \end{align}
Here, the last inequality follows from $0<\consttwo, c^{\prime}  \le 1/30$. \\

{\bf Putting everything together:}
Using the decomposition \eqref{ineq:intern42} and combining the bounds for $\normf{\VXXT \MM_1}$ and for $\normf{\MM_i}$ for $ 2 \le i \le 6 $, we obtain
    \begin{align*}
        &\normf{\VXXT \left( \XXtplus - \XXtwplus \right)}\\
        &\le 
        \left(1-\frac{15\mu}{16}\right)\normf{\VXXT\left(\XXt-\XXtw\right)}+\frac{5\constone\mu}{4}\normf{\left(\XXt-\XXtw\right)\WXX}\\
        &+\frac{2\sqrt{2}\left(\consttwo+\constthree\right) }{1-\consttwo}\normf{\XXt-\XXtw}\mu  +\left(8 c^{\prime} +\frac{2\sqrt{2}\constfive}{1-\consttwo} \right) \normf{\XXt - \XXtw}\mu \\
        &+\mu^2\normf{\XXt-\XXtw}+4c^{\prime}\mu \specnorm{\XXstar - \XXt}+5\mu^2\normf{\XXt - \XXtw}+\mu^2\specnorm{\XXstar-\XXt}\\ 
        &\overleq{(a)} \left(1-\frac{3\mu}{4}+\frac{5\mu}{2} \left( \frac{\sqrt{2}\left(\consttwo+\constthree\right) }{1-\consttwo}+4 c^{\prime} +\frac{\sqrt{2}\constfive}{1-\consttwo}  \right)\right)\normf{\VXXT\left(\XXt-\XXtw\right)}
       \\
        & +\frac{5\mu}{4} \left( \constone+\frac{1}{5}+\frac{2\sqrt{2}\left(\consttwo+\constthree\right) }{1-\consttwo}+8 c^{\prime} +\frac{2\sqrt{2}\constfive}{1-\consttwo}  \right)\normf{\left(\XXt-\XXtw\right)\WXX} +\frac{1}{18}\mu \specnorm{\XXstar-\XXt} \\
        &\le 
        \left(1-\frac{5\mu}{8}\right)\normf{\VXXT\left(\XXt-\XXtw\right)}+\frac{3\mu}{8}\normf{\left(\XXt-\XXtw\right)\WXX}+\frac{1}{18}\mu \specnorm{\XXstar-\XXt},
    \end{align*}
    where $(a)$ uses Lemma \ref{lemma:auxcloseness1}, assumption \eqref{assump:closeness8}, and the step size $\mu\le \frac{1}{32}$; the  last inequality follows by choosing $\consttwo,\constthree,c^{\prime}\le \frac{1}{360}$ and $\constone \le \frac{1}{20}$. By symmetry, the same reasoning yields
\begin{align*}
&\normf{\left( \XXtplus - \XXtwplus \right)\WXX}\\
    &\le 
        \left(1-\frac{5\mu}{8}\right)\normf{\left(\XXt-\XXtw\right)\WXX}+\frac{3\mu}{8}\normf{\VXXT\left(\XXt-\XXtw\right)}+\frac{1}{18}\mu \specnorm{\XXstar-\XXt}.
\end{align*}
Summing the two bounds gives
\begin{align*}
    &\normf{\VXXT \left( \XXtplus - \XXtwplus \right)}+\normf{\left( \XXtplus - \XXtwplus \right)\WXX}\\
    &\le 
    \left(1-\frac{\mu}{4}\right)\left(\normf{\VXXT\left(\XXt-\XXtw\right)}+\normf{\left(\XXt-\XXtw\right)\WXX}\right)+\frac{1}{9}\mu \specnorm{\XXstar-\XXt}.
\end{align*}

Finally, we prove \eqref{eq:advance estimate1}. Using the expressions \eqref{eq:Xt1} and \eqref{eq:Xt1uv} again, we have
\begin{align*}
     &\XXtplus  -  \XXtwplus\\
      &=
    \XXt -\XXtw
            + \mu \MM_t \VRRt\VRRtP
            -\mu \MM_{t}^{\left(\ww,\vv \right)} \VRRtw\VRRtwP 
            - \mu  \VLLtw\VLLtwP\MM_{t}^{\left(\ww,\vv \right)}
              \\
            & 
            +\mu  \VLLt\VLLtP\MM_t+ \mu^2 \MM_t\VRRt \SSigma_{t}^{-1} \VLLtP\MM_t
            - \mu^2\MM_{t}^{\left(\ww,\vv \right)}\VRRtw \left({\SSigma_{t}^{\left(\ww,\vv \right)}}\right)^{-1} \VLLtwP\MM_{t}^{\left(\ww,\vv \right)}  \\
    &=
    \XXt - \XXtw
    +\mu  \MM_t (\VRRt\VRRtP-\VRRtw\VRRtwP ) +\mu( \VLLt\VLLtP-\VLLtw\VLLtwP) \MM_t  \\
    & +\mu  \left( \MM_t - \MM_{t}^{\left(\ww,\vv \right)} \right) \VRRtw\VRRtwP   +\mu  \VLLtw\VLLtwP\left( \MM_t - \MM_{t}^{\left(\ww,\vv \right)} \right) \\
    &+\mu^2 
    \left(\MM_t\VRRt \SSigma_{t}^{-1} \VLLtP\MM_t-\MM_{t}^{\left(\ww,\vv \right)}\VRRtw \left({\SSigma_{t}^{\left(\ww,\vv \right)}}\right)^{-1} \VLLtwP\MM_{t}^{\left(\ww,\vv \right)}\right).
    \end{align*}
    Applying the triangle inequality and Lemma~\ref{lemma:projection matrices} together with \eqref{eq:M6} yields
\begin{align*}
     &\normf{\XXtplus - \XXtwplus}\\
     &\overleq{(a)} \normf{\XXt - \XXtw}+\frac{2\sqrt{2}\mu \specnorm{\MM_t }}{\left(1-\consttwo\right)\sigma_{\min} (\XXstar)}\normf{\XXt-\XXtw} +\mu \normf{\left( \MM_t - \MM_{t}^{\left(\ww,\vv \right)} \right) \VRRtw}\\
     &+\mu \normf{\VLLtwP \left( \MM_t - \MM_{t}^{\left(\ww,\vv \right)} \right)}
     +\mu^2\left(5\normf{\XXt- \XXtw}+\specnorm{\XXstar-\XXt}\right)\\
     &\overleq{(b)}\left(1+5\mu^2+\frac{2\sqrt{2}\left(\consttwo+\constfive\right)\mu}{1-\consttwo}+2\mu \left(4c^{\prime}+1\right)\right)\normf{\XXt - \XXtw}\\
     &+\left(4c^{\prime}\mu +\mu^2\right)\specnorm{\XXstar-\XXt}
     \\
      &\le 
      \frac{\sigma_{\min} (\XXstar)}{80}.
   \end{align*}
   where (a) uses Lemma~\ref{lemma:projection matrices} and \eqref{eq:M6}, (b) uses \eqref{ineq:weakboundintern9}, \eqref{ineq:weakboundintern7}, and \eqref{ineq:weakboundintern7_2}, and the last inequality follows from assumptions \eqref{assump:closeness7}–\eqref{assump:closeness8} with $\consttwo,\constthree,c^{\prime}\le \frac{1}{360}$. This establishes \eqref{eq:advance estimate1} and completes the proof of Lemma~\ref{lemma:auxsequencecloseness}.
    \end{proof}

\subsection{Proof of Lemma \ref{lemma:localconv}}\label{sec:prooflocalconv_b}
In order to prove Lemma \ref{lemma:localconv}, we need the follow auxiliary result, which  bounds $ \specnorm{ \XXt -\XXstar } $ in terms of $\specnorm{\VXXT \left( \XXt -\XXstar \right) } $ and $\specnorm{\left( \XXt -\XXstar \right)\WXX}$. 

\begin{lemma}\label{lemma:localconvaux}
Assume that 
\begin{align}
   \max\{ \specnorm{\VXXPT \VLLt}, \specnorm{\WXXPT \VRRt}\} \le \frac{1}{\sqrt{2}}.\label{ineq:localconv1}
\end{align}    
Then 
\begin{align}
    \specnorm{\VXXPT \XXt \WXXP}
    &\le 
    2   \specnorm{\VXXPT \VLLt}
    \specnorm{ \VXXT  \left( \XXt - \XXstar \right) \WXXP},\notag \\
      \specnorm{\VXXPT \XXt \WXXP}
       &\le 2\specnorm{ \WXXPT \VRRt}
        \specnorm{ \VXXPT  \left( \XXt- \XXstar \right) \WXX}.
    \label{ineq:localconv2}
\end{align}
Moreove,
\begin{align}\label{ineq:localconv3}
    &\specnorm{ \XXt - \XXstar }\\
    &\le 
    \left( 1+ \specnorm{ \WXXPT \VRRt } \right)
        \specnorm{ \left( \XXt - \XXstar \right)\WXX } \notag
        +  \left( 1+ \specnorm{ \VXXPT \VLLt } \right) \specnorm{ \VXXT\left( \XXt - \XXstar \right)}.
\end{align}
\end{lemma}
\begin{proof}
Let the compact SVD of $\XXt$ be $\VLLt\SSigma_{t}\VRRt^\top $. Then 
    \begin{align*}
        \VXXPT \XXt \WXXP
        &=
        \VXXPT  \VLLt \VLLtP\XXt\WXXP\\
        &=
        \VXXPT \VLLt\left(  \VXXT \VLLt \right)^{-1} \VXXT\VLLt \VLLtP \XXt \WXXP\\
        &=
        \VXXPT \VLLt\left(  \VXXT \VLLt \right)^{-1} \VXXT  \XXt \WXXP\\
        &=
        \VXXPT \VLLt\left(  \VXXT \VLLt \right)^{-1} \VXXT  \left( \XXt - \XXstar \right) \WXXP   ,
    \end{align*} 
    due to $\XXstar\WXXP=0$. Hence, 
    \begin{align*}
       \specnorm{\VXXPT \XXt \WXXP}
       &\le 
       \specnorm{\VXXPT \VLLt}  \specnorm{ \left(  \VXXT \VLLt \right)^{-1}}
       \specnorm{ \VXXT  \left( \XXt - \XXstar \right) \WXXP}\\
       &=\frac{\specnorm{\VXXPT \VLLt}}{\sigma_{\min} \left( \VXXT \VLLt \right)}\specnorm{ \VXXT  \left( \XXt- \XXstar \right) \WXXP}\\
       &\le 
       2\specnorm{\VXXPT \VLLt}
        \specnorm{ \VXXT  \left( \XXt- \XXstar \right) \WXXP},
    \end{align*}
   where we use the fact that  $\sigma^2_{\min} \left( \VXXT \VLLt \right) =  1- \specnorm{\VLLtP \VXXP }^2$ and assumption \eqref{ineq:localconv1} in the last inequality. Following a similar argument, we can get 
    \begin{align*}
       \specnorm{\VXXPT \XXt \WXXP}
       \le 2\specnorm{ \WXXPT \VRRt}
        \specnorm{ \VXXPT  \left( \XXt- \XXstar \right) \WXX} .
    \end{align*}
   For \eqref{ineq:localconv3}, observe that
    \begin{align*}
        &\specnorm{\XXt - \XXstar } \\&\overleq{(a)}
        \specnorm{ \VXXT \left( \XXt - \XXstar \right) }+
        \specnorm{ \VXXPT \left( \XXt - \XXstar \right)}\\
        &\overleq{(b)}
        \specnorm{ \VXXT \left( \XXt - \XXstar \right) }+
        \specnorm{ \VXXPT \left( \XXt - \XXstar \right)\WXX}+\specnorm{ \VXXPT \left( \XXt - \XXstar \right)\WXXP}\\
        &\le \left( 1+ \specnorm{ \WXXPT \VRRt } \right)
        \specnorm{ \left( \XXt - \XXstar \right)\WXX }+  \left( 1+ \specnorm{ \VXXPT \VLLt } \right) \specnorm{ \VXXT\left( \XXt - \XXstar \right)},
    \end{align*}  
    where the inequality $(a)$ follows from $\VXX \VXXT+\VXXP \VXXPT=\II_{n_1}$, the  inequality $(b)$ comes from the identity $\WXX \WXXT+\WXXP \WXXPT=\II_{n_2}$, and the last inequality arises from \eqref{ineq:localconv2}.     This completes the proof.
    \end{proof}

    \begin{proof}[Proof of Lemma \ref{lemma:localconv}]
          For simplicity, denote
        \begin{equation*}
            \MM_t 
            :=
            \left(\Aops\right) \left( \XXstar - \XXt  \right)
            =
            \XXstar - \XXt
            + \bracing{=: \EEb_t}{\left( \Aops - \IdOp \right) \left(\XXstar - \XXt \right)}.
        \end{equation*}
According to \eqref{eq:Xt1}, one has
 \begin{eqnarray}
&&  \XXstar - \XXtplus  = \XXstar - \XXt - \mu \left( \XXstar - \XXt\right)\VRRt\VRRtP
- \mu  \VLLt\VLLtP \left( \XXstar -\XXt\right) \notag \\
&&- \mu \EEb_t \VRRt\VRRtP - \mu\VLLt\VLLtP\EEb_t -\mu^2\MM_t\VRRt \SSigma_{t}^{-1} \VLLtP\MM_t \notag \\ 
&&= 
\left(\II_{n_1} - \mu \VLLt\VLLtP \right) \left(\XXstar - \XXt \right) \left( \II_{n_2} -\mu \VRRt\VRRtP \right)
-\mu^2 \VLLt\VLLtP \left( \XXstar - \XXt \right) \VRRt\VRRtP \notag \\
& &- \mu \EEb_t \VRRt\VRRtP - \mu \VLLt\VLLtP \EEb_t -\mu^2\MM_t\VRRt \SSigma_{t}^{-1} \VLLtP\MM_t.  \label{eq:fenjie13}
\end{eqnarray}
Multiplying by $ \VXXT$ and inserting $\VXX \VXXT+\VXXP \VXXPT=\II_{n_1}$ gives 
 \begin{align*}
            \VXXT \left( \XXstar - \XXtplus \right)
            &= \VXXT \left(\II_{n_1} - \mu \VLLt\VLLtP\right) \VXX \VXXT \left(\XXstar - \XXt\right) \left( \II_{n_2} -\mu \VRRt\VRRtP \right) \\
            &+ \VXXT\left(\II_{n_1} - \mu \VLLt\VLLtP \right) \VXXP \VXXPT \left(\XXstar - \XXt\right) \left( \II_{n_2} -\mu \VRRt\VRRtP \right)\\
            &-\mu^2 \VXXT \VLLt\VLLtP \left( \XXstar - \XXt \right) \VRRt\VRRtP
            - \mu \VXXT\EEb_t \VRRt\VRRtP - \mu \VXXT\VLLt\VLLtP \EEb_t \\
            &-\mu^2 \VXXT\MM_t\VRRt \SSigma_{t}^{-1} \VLLtP\MM_t\\
            &= (I)+\mu \cdot (II)- (III)-\mu^2 \cdot(IV),
 \end{align*}
   where     
 \begin{eqnarray*}
 (I): &=& \left(\II_{n_1} - \mu \VXXT\VLLt\VLLtP \VXX\right)  \VXXT \left(\XXstar - \XXt\right) \left( \II_{n_2}-\mu \VRRt\VRRtP \right)\\
 (II): &=& \VXXT \VLLt\VLLtP \VXXP \VXXPT \XXt \left( \II_{n_2} -\mu \VRRt\VRRtP \right) \\
(III): &=&  \left(\mu^2  \VXXT \VLLt\VLLtP \left( \XXstar - \XXt \right) \VRRt\VRRtP
            + \mu \VXXT\EEb_t \VRRt\VRRtP + \mu \VXXT\VLLt\VLLtP\EEb_t\right) \\
   (IV): &=&        \VXXT\MM_t\VRRt \SSigma_{t}^{-1} \VLLtP\MM_t
 \end{eqnarray*}

        We next estimate the spectral norm of these terms separately.
        
        \paragraph*{Estimating term $(I)$:}
       A simple calculation gives
        \begin{align*}
             \specnorm{(I)}= &\specnorm{
            \left(\II_{n_1} - \mu \VXXT\VLLt\VLLtP \VXX\right)  \VXXT \left(\XXstar - \XXt\right) \left( \II_{n_2} -\mu \VRRt\VRRtP \right)
            }\\
            &\le 
            \specnorm{\II_{n_1} - \mu \VXXT\VLLt\VLLtP \VXX} 
            \specnorm{\VXXT \left(\XXstar - \XXt\right)}
            \specnorm{\II_{n_2} -\mu \VRRt\VRRtP}\\
            &\overleq{(\textup{i})}
            \specnorm{\II_{n_1} - \mu \VXXT\VLLt\VLLtP \VXX} 
            \specnorm{\VXXT \left(\XXstar - \XXt\right)}\\
            &=  \left( 1 -\mu  \sigma^2_{\min} (\VXXT \VLLt) \right)
            \specnorm{\VXXT \left(\XXstar - \XXt \right)}\\
            &\overleq{(\textup{ii})}
            \left( 1 - \frac{15\mu}{16}   \right)
            \specnorm{\VXXT \left(\XXstar - \XXt \right)}.
        \end{align*}
        Here, $(\textup{i})$ uses   that the eigenvalues of $\VRRt\VRRtP$ are $0$ or $1$ and $0< \mu \le \frac{1}{15}$, and $(\textup{ii})$ follows from the assumption \eqref{ineq:localconv8} and the fact $\sigma^2_{\min} (\VXXT \VLLt)=1-\sigma^2_{\max} (\VXXPT \VLLt)$.
        

        \paragraph*{Estimating term $(II)$:}
        Note that 
        \begin{align*}
              \specnorm{(II)}=&\specnorm{\VXXT \VLLt\VLLtP \VXXP \VXXPT \XXt \left( \II_{n_2} -\mu \VRRt\VRRtP \right)}\\
            &\le 
            \specnorm{\VXXT\VLLt } \specnorm{\VLLtP\VXXP }
            \specnorm{\VXXPT \left(\XXt -\XXstar \right)}
            \specnorm{\II_{n_2} -\mu \VRRt\VRRtP}\\
            &\overleq{(a)}
             \frac{1}{8} \specnorm{\VXXPT \left(\XXt -\XXstar \right)}\\
            &\le
             \frac{1}{8}\left( \specnorm{\VXXPT \left( \XXt -\XXstar \right) \WXX\WXXT} 
             + \specnorm{\VXXPT \left( \XXt -\XXstar \right) \WXXP\WXXPT}\right)\\
            &\le
             \frac{1}{8}
            \left( \specnorm{\left( \XXt -\XXstar \right)\WXX} 
             + \specnorm{\VXXPT \left( \XXt -\XXstar \right) \WXXP}\right)\\
            &\overleq{(b)}
            \frac{1}{8}\specnorm{ \left( \XXt -\XXstar \right)\WXX} 
            +  \frac{1}{4}\specnorm{ \VXXPT \VLLt }
             \specnorm{\VXXT \left( \XXt -\XXstar \right)}\\
            &\leq
            \frac{1}{8}\specnorm{ \left( \XXt -\XXstar \right)\WXX}+\frac{1}{32} 
             \specnorm{\VXXT \left( \XXt -\XXstar \right)},
        \end{align*}
       where the  inequality $(a)$ follows from  the assumption  $\specnorm{\VLLtP\VXXP }\le \frac{1}{8}$ and $\specnorm{\VXXT\VLLt }\le 1$, and the  inequality $(b)$ arises from Lemma \ref{lemma:localconvaux}.

        \paragraph*{Estimating term $(III)$:}
        \begin{align*}
            \specnorm{(III)}
            &\le 
            \mu^2 \specnorm{\XXstar - \XXt}
            +  \mu \left(\specnorm{\EEb_t \VRRt }+\specnorm{\VLLtP\EEb_t }\right)\\
            &\le
             \frac{9}{8} \mu^2 \left(\specnorm{\VXXT \left(\XXstar - \XXt \right)} +\specnorm{\left(\XXstar - \XXt \right)\WXX}\right)+2\mu\specnorm{\EEb_t},
        \end{align*}
        where the last inequality follows from Lemma \ref{lemma:localconvaux} and the fact that  $\VRRt$ and $\VLLt$ are orthonormal.
        
        \paragraph*{Estimating term $(IV)$:}
       From the definition of $\MM_t$, one has 
        \begin{align}
        \specnorm{\MM_t \VRRt}
         &\overleq{(\textup{i})}
        \specnorm{\XXstar -\XXt} 
        + \specnorm{ \EEb_t \VRRt } 
        \nonumber \\
        &\overleq{(\textup{ii})} 
        \frac{9}{8} \left(\specnorm{\VXXT \left(\XXstar - \XXt \right)} +\specnorm{\left(\XXstar - \XXt \right)\WXX}\right) 
        + \specnorm{\EEb_t },
        \label{ineq:localconv7}
        \end{align}
        where $(\textup{i})$ follows from the triangle inequality and
        $(\textup{ii})$ follows from Lemma \ref{lemma:localconvaux}.
       Hence
        \begin{align*}
            \specnorm{(IV)}
            &\le 
            \mu^2 \specnorm{\MM_t \VRRt} \specnorm{\SSigma_{t}^{-1} \VLLtP}\specnorm{\MM_t}\\
            &\overleq{(a)}
             \mu^2 \frac{\sigma_{\min}\left( \XXstar \right)}{\sigma_{\min}\left( \XXt \right)}   \specnorm{\MM_t \VRRt}  \\
            &\overleq{(b)}
              \frac{9}{4}\mu^2\left(\specnorm{\VXXT \left(\XXstar - \XXt \right)} +\specnorm{\left(\XXstar - \XXt \right)\WXX}\right) + 2\mu^2  \specnorm{\EEb_t}.  
        \end{align*}
        where $(a)$ uses  \eqref{ineq:weakboundintern9}, and  $(b)$ uses \eqref{ineq:localconv7} together with ${\sigma_{\min}\left( \XXt \right)}\ge {\sigma_{\min}\left( \XXstar \right)}/{2}$, which comes from Weyl’s inequality
        $\left|\sigma_{\min}\left(\XXt \right)-\sigma_{\min}\left( \XXstar \right)\right| \le \specnorm{\XXt-\XXstar} \le c_2\sigma_{\min}\left( \XXstar \right)$.

         \paragraph*{Combining the bounds:} Aggregating the four estimates,
        \begin{align*}
            &\specnorm{
                \VXXT \left( \XXstar - \XXtplus \right)     }\\
                &\le
            \left( 1 - \frac{29\mu  }{32}+\frac{27\mu^2}{8} \right)
            \specnorm{\VXXT \left( \XXstar -\XXt \right)}
            +\frac{7\mu^2}{2}\specnorm{\left( \XXstar -\XXt \right)\WXX}
            +
            2\left(1+\mu\right) \mu  \specnorm{\EEb_t}\\
            &\le
            \left( 1 - \frac{25\mu  }{44} \right)
            \specnorm{\VXXT \left( \XXstar - \XXt \right)}
            +\frac{7\mu}{22}\specnorm{\left( \XXstar -\XXt \right)\WXX}
            +
            3 \mu  \specnorm{\EEb_t },
        \end{align*}
        where the last inequality follows $0<\mu\le\frac{1}{15}$. By symmetry, the same argument yields
\[
            \specnorm{
                \left( \XXstar - \XXtplus \right)\WXX
            } \le
            \left( 1 - \frac{25\mu  }{44} \right)
            \specnorm{\left( \XXstar - \XXt \right)\WXX}
            +\frac{7\mu}{22}\specnorm{\VXXT\left( \XXstar -\XXt \right)} 
            +
            3 \mu  \specnorm{\EEb_t }.
\]
     Adding these two inequalities, 
        \begin{align*}
            \specnorm{
                \VXXT \left( \XXstar - \XXtplus \right)}+&
                \specnorm{
                \left( \XXstar - \XXtplus \right)\WXX}\\
                \le& 
                 \left( 1 - \frac{\mu  }{4} \right)\left(
            \specnorm{\VXXT\left( \XXstar - \XXt \right)} 
            +\specnorm{\left( \XXstar -\XXt \right)\WXX}\right)
            +
            6 \mu  \specnorm{\EEb_t }
        \end{align*}  
        
        Finally, we turn to  prove the inequality \eqref{eq:advance estimate2}.   According to \eqref{eq:fenjie13} and applying the triangle inequality, we obtain
    \begin{align*}
    \specnorm{\XXstar - \XXtplus}
    &\le 
    \specnorm{\II_{n_1} - \mu \VLLt\VLLtP} 
    \specnorm{\XXstar - \XXt} 
    \specnorm{\II_{n_2} -\mu \VRRt\VRRtP }
    +\mu^2  \specnorm{\XXstar - \XXt}\\
    &+ 2 \mu \specnorm{\EEb_t} 
    + \mu^2 \specnorm{\MM_t}^2 \specnorm{\SSigma_{t}^{-1}}\\
    &\overleq{(a)}
     \specnorm{\XXstar - \XXt}
    +   \mu^2 c_2  \sigma_{\min} \left(\XXstar\right)  
    + 2 \mu c_5  \sigma_{\min} (\XXstar)+  (c_2+c_5)^2 \mu^2 \frac{\sigma^2_{\min} (\XXstar)}{\sigma_{\min} (\XXt)}  \\
    &\overleq{(b)}   \left(c_2+ c_2\mu^2+2\mu c_5+2(c_2+c_5)^2\mu^2\right) \sigma_{\min}(\XXstar)\\
    &\leq  \frac{\sigma_{\min}(\XXstar)}{80}.
\end{align*}
Here, the inequality $(a)$ follows from  assumptions \eqref{ineq:localconv5}, \eqref{ineq:localconv4},\eqref{ineq:weakboundintern9}  together with $\specnorm{\II_{n_1}- \mu \VLLt\VLLtP}\leq 1$, and the inequality $(b)$ comes from ${\sigma_{\min}\left( \XXt \right)}\ge {\sigma_{\min}\left( \XXstar \right)}/{2}$. The last inequality holds for $ \mu \le \frac{1}{15} $ and $\consttwo, \constfive\le 0.01$. This proves \eqref{eq:advance estimate2} and completes the proof.
        \end{proof}

\subsection{Proof of Lemma \ref{Le:4_6}} \label{pf:Le4.6}
\begin{proof}
First, inequality \eqref{assumption:initiation} follows directly from Lemma \ref{lemma:spectralinitialization}. Specifically, by Lemma \ref{lemma:spectralinitialization}, when $m \ge C \left(n_1+n_2\right)r \kappa^2$ for some constant $C>0$,  it holds with probability at least $1- 4\exp(-\left(n_1+n_2\right))$ that $ \GG_{0} \le 2c_0 \sigma_{\min} (\XXstar)$. By the triangle inequality, this immediately implies  $\GG_{0,\star} \le 2  \GG_{0} $.

Next, \eqref{eq:Gtstar} and \eqref{eq:Gt} are proved by induction. From \eqref{assumption:initiation}, both inequalities hold for $t=0$. 
Assume that  \eqref{eq:Gtstar} and \eqref{eq:Gt} hold for $t$-th iteration, we next prove that they also hold for $(t+1)$-th iteration.  Under the induction hypothesis, 
\[
 \GG_{t,\star}\le 2\left(1-\frac{\mu}{10} \right)^t c_0 \sigma_{\min} (\XXstar) \quad\mbox{and} \quad  \GG_t \le 3\left(1-\frac{\mu}{10} \right)^{t} c_0 \sigma_{\min} (\XXstar),
\]
so, by the definitions \eqref{eq:Gt0} and \eqref{eq:Gtstar0},
\begin{equation} \label{eq:G_t1}
 \specnorm{ \XXstar - \XXt } \le  3 c_0 \sigma_{\min} (\XXstar),   \quad
    \supw  \normf{\XXt - \XXtw} \le 3 c_0 \sigma_{\min} (\XXstar)
\end{equation}
and
\begin{align}
\specnorm{\VXXT\left(\XXstar - \XXt \right)}&+\specnorm{\left(\XXstar - \XXt \right)\WXX} \le  2 c_0 \sigma_{\min} (\XXstar)  \label{crucial condition}\\
    \supw  \normf{\VXXT\left(\XXt - \XXtw\right)} &+\supw  \normf{\left(\XXt - \XXtw\right)\WXX} \le 2 c_0\sigma_{\min} (\XXstar).   \notag
\end{align}
Therefore,  Lemma \ref{lemma:Davis-Kahan inequality}  gives
\begin{eqnarray}\label{eq:VV}
   \max\{ \specnorm{\VXXPT \VV_{t}}, \specnorm{\WXXPT \WW_{t}}\}
    &\le & 
    \frac{\sqrt{2}\left(\specnorm{ \VXXT \left( \XXstar - \XXt\right)} +\specnorm{ \left( \XXstar - \XXt\right)\WXX} \right)}{\sigma_{\min} \left( \XXstar\right)}  \notag \\
   & \le & 2\sqrt{2}  c_0 ,
\end{eqnarray}
where the last inequality follows from \eqref{crucial condition}.  On the one hand, for each $\left(\ww, \vv\right) \in \mathcal{N}$, applying Lemma \ref{lemma:auxsequencecloseness} with conditions \eqref{eq:G_t1}, \eqref{eq:VV} and $c_0\le \frac{1}{1080}$, we obtain
\begin{align}\label{eq:G 8 2}
    & \normf{\VXXT \left( \XXtplus - \XXtwplus \right)}+\normf{ \left( \XXtplus - \XXtwplus \right)\WXX}\\  
   & \le  \left(1-\frac{\mu}{4}\right)\left(\normf{\VXXT\left(\XXt-\XXtw\right)}+\normf{\left(\XXt-\XXtw\right)\WXX}\right)+\frac{1}{9}\mu \specnorm{\XXstar-\XXt} \notag \\
   & \le \left(1-\frac{\mu}{4}\right)\left(\normf{\VXXT\left(\XXt-\XXtw\right)}+\normf{\left(\XXt-\XXtw\right)\WXX}\right) \notag \\
    & +\frac{\mu}{8} \left(\specnorm{\VXXT\left(\XXt-\XXstar\right)}+\specnorm{\left(\XXt-\XXstar\right)\WXX}\right),  \notag
\end{align}
where  the last inequality follows from Lemma \ref{lemma:localconvaux}.  On the other hand, according to Lemma \ref{lemma:key}, one has
\begin{align}
       &  \specnorm{ \left( \Aops -\IdOp \right) \left( \XXstar - \XXt \right)  }
        \le 
        4c^{\prime} \specnorm{ \XXstar - \XXt }+6c^{\prime} \supw  \normf{\XXt-\XXtw} \notag \\
        & \le 30 c_0c^{\prime}\sigma_{\min} (\XXstar), \label{eq:condition_3}
    \end{align}
    where the last inequality follows from \eqref{eq:G_t1}. Thus, all assumptions of Lemma \ref{lemma:localconv} are satisfied in view of \eqref{eq:G_t1}, \eqref{eq:condition_3}, and \eqref{eq:VV}, with  $c_0, c^{\prime} \le \frac{1}{1080}$.  Applying Lemma \ref{lemma:localconv}, we have

\begin{align}
    &\specnorm{\VXXT \left( \XXtplus - \XXstar \right)}+\specnorm{\left( \XXstar - \XXtplus \right)\WXX} \label{eq:G 8 1} \\
        &\le 
        \left( 1 - \frac{\mu  }{4} \right)\left(
            \specnorm{\VXXT\left( \XXstar - \XXt \right)}
            +\specnorm{\left( \XXstar -\XXt \right)\WXX}\right)
        +6 \mu \specnorm{ \left(\Aops - \IdOp \right)
        \left(\XXstar - \XXt\right)}  \notag \\
        &\overleq{(\textup{i})}
        \left( 1 - \frac{\mu  }{4} \right)\left(
            \specnorm{\VXXT\left( \XXstar - \XXt \right)}
            +\specnorm{\left( \XXstar -\XXt \right)\WXX}\right)+24c^{\prime}\mu   \specnorm{ \XXstar - \XXt }  \notag \\
        &+36c^{\prime}\mu \supw  \normf{\XXt - \XXtw}  \notag
\end{align}        
\begin{align*}
        &\overleq{(\textup{ii})}
        \left(1 - \frac{\mu}{4}+27c^{\prime}\mu \right) 
        \left(
            \specnorm{\VXXT\left( \XXstar - \XXt \right)}
            +\specnorm{\left( \XXstar -\XXt \right)\WXX}\right)  \notag \\
        &+45c^{\prime}\mu  \left( \supw \normf{\VXXT\left(\XXt - \XXtw\right)}+\supw \normf{\left(\XXt - \XXtw\right)\WXX} \right)  \notag \\
        &\le 
        \left(1 - \frac{9\mu}{40}\right)\left(
            \specnorm{\VXXT\left( \XXstar - \XXt \right)}
            +\specnorm{\left( \XXstar -\XXt \right)\WXX}\right)  \notag \\
        &+\frac{\mu }{10}\left( \supw \normf{\VXXT\left(\XXt - \XXtw\right)}+\supw \normf{\left(\XXt - \XXtw\right)\WXX} \right),  \notag
\end{align*}
where the inequality $(\textup{i})$ comes from Lemma \ref{lemma:key},  the inequality $(\textup{ii})$ arises from Lemma \ref{lemma:auxcloseness1} and \ref{lemma:localconvaux}, and the last inequality follows from the fact that  $c^{\prime}\le \frac{1}{1080}$.
Taking the supremum over \eqref{eq:G 8 2} and combining  \eqref{eq:G 8 1}, we obtain
\begin{align*}
    \GG_{t+1,\star} &\le \left(1-\frac{3\mu}{20}\right)\left( \supw\normf{\VXXT\left(\XXt-\XXtw\right)}+ \supw\normf{\left(\XXt-\XXtw\right)\WXX}\right)\\
    &+\left(1-\frac{\mu}{10} \right) \left(\specnorm{\VXXT\left(\XXt-\XXstar\right)}+\specnorm{\left(\XXt-\XXstar\right)\WXX}\right)
    \\
    &\le \left(1-\frac{\mu}{10} \right) \GG_{t,\star}\\
    &\le \left(1-\frac{\mu}{10} \right)^{t+1}2 c_0 \sigma_{\min} (\XXstar).
\end{align*}
Moreover, Lemma \ref{lemma:localconv} also implies
 \[
  \specnorm{\XXtplus-\XXstar} \le  \frac{\sigma_{\min} (\XXstar)}{80}.
 \]
Applying Lemma \ref{lemma:Davis-Kahan inequality} again,
\begin{align*}
   &\max\{ \specnorm{\VXXPT \VV_{t+1}}, \specnorm{\WXXPT \WW_{t+1}}\}\\
    &\le
    \frac{\sqrt{2}\left(\specnorm{ \VXXT \left( \XXstar - \XXtplus\right)} +\specnorm{ \left( \XXstar - \XXtplus\right)\WXX} \right)}{\sigma_{\min} \left( \XXstar\right)}
    \le 2\sqrt{2}  c_0 ,
\end{align*}
where the last inequality uses
\begin{equation*}
    \specnorm{ \VXXT \left( \XXstar - \XXtplus \right)}+\specnorm{ \left( \XXstar - \XXtplus \right)\WXX}\le \GG_{t+1,\star}\le 2c_0 \sigma_{\min} (\XXstar).
\end{equation*}
Finally, combining Lemmas \ref{lemma:auxcloseness1} and \ref{lemma:localconvaux} with \eqref{eq:advance estimate1} from Lemma \ref{lemma:auxsequencecloseness}, one obtains that 
\begin{align*}
    \GG_{t+1} \le \frac{3\GG_{t+1,\star}}{2}\le 3\left(1-\frac{\mu}{10} \right)^{t+1} c_0 \sigma_{\min} (\XXstar)
\end{align*}
holds with probability at least $1- 6\exp(-\left(n_1+n_2\right))$.
This completes the induction step for iteration $t+1$. 
\end{proof}

\section{Auxiliary Lemmas}
The following lemma shows that, at initialization, both the original iterates and all virtual iterates are close to the ground truth
$\XXstar$. 

\begin{lemma}\cite[Lemma~6]{caifast} \label{lemma:spectralinitialization}
For any constant $c_0 >0$, there exists an absolute constant $C>0$  such that if $m \ge C\kappa^2r\left(n_1+n_2\right)$, then with probability at least $1-4\exp(-\left(n_1+n_2\right))$,
\begin{align}
    \specnorm{\XX_0-\XXstar} \le c_0 \sigma_{\min}(\XXstar) \quad  \aand \quad
    \specnorm{\XX_0-\XX_{0}^{\left(\ww, \vv\right)}} \le c_0 \sigma_{\min}(\XXstar).
\end{align}
\end{lemma}

\begin{lemma}\cite[Lemma~24]{JMLR:v22:20-1067}\label{lemma:procrustebound}
Let $\XXt, \XXstar \in \R^{n_1 \times n_2}$ be rank-$r$ matrices. Then
\begin{align*}
    \dist \left( \XXt, \XXstar \right)
    \le
    \sqrt{\sqrt{2}+1}
    \normf{\XXt - \XXstar},
\end{align*}
where $\dist\left( \XXt, \XXstar \right)$ is defined in \eqref{eq:def_dist}.
\end{lemma}

\begin{lemma}\cite[Corollary~2.8]{chen2021spectral}\label{lemma:Davis-Kahan inequality}
Let  $\BB_1, \BB_2 \in \mathbb{R}^{n_1 \times n_2}$ be two matrices with full SVD $\BB_1=\VV_1\SSigma_{\BB_1}\WW_1^\top$ and $\BB_2=\VV_2\SSigma_{\BB_2}\WW_2^\top$ respectively. Let  $ \VV_{1,r} \in \R^{ n_1 \times r }$( resp. $\WW_{1,r}$) contain the first $r$ columns of $\VV_1$, and let $ \VV_{1,r,\bot} \in \R^{n_1 \times (n_1-r)}$( resp. $\WW_{1,r,\bot}$) contain the remaining $n_1-r$ columns. Suppose the singular value of $\BB_1$ satisfy $\left|\lambda_r\left(\BB_1\right)\right|>\left|\lambda_{r+1}\left(\BB_1\right)\right|$ and 
\begin{equation*}
    \specnorm{\BB_1-\BB_2} \le \left( 1-\frac{1}{\sqrt{2}}\right)\left(\left|\lambda_r\left(\BB_1\right)\right|-\left|\lambda_{r+1}\left(\BB_1\right)\right|\right).
\end{equation*}
Then
\begin{equation*}
   \max\{ \specnorm{ \VV_{1,r,\bot}^\top \VV_{2}}, \specnorm{ \WW_{1,r,\bot}^\top \WW_{2}}\}\} \le \frac{\sqrt{2}\left(\specnorm{\VV_{1,r}^{\top}\left( \BB_1-\BB_2\right)}+\specnorm{\left( \BB_1-\BB_2\right)\WW_{1,r}}\right)}{\left|\lambda_r\left(\BB_1\right)\right|-\left|\lambda_{r+1}\left(\BB_1\right)\right|}
\end{equation*} 
\end{lemma}

The following lemma bounds the distance between the projection matrices onto the singular subspaces of two rank-$r$ matrices.

\begin{lemma}\label{lemma:projection matrices}
\cite[Lemma~4.2]{wei2016guarantees} 
Let  $\XXt$  and $\XX$ be rank-$r$ matrices with compact SVDs  $\XXt=\VV_t\SSigma_{t}\WW_t^\top$ and $\XX=\VV\SSigma\WW^\top$, respectively. Then
\begin{align*}
    \specnorm{\VV_t\VV_t^\top-\VV\VV ^\top} &\le \frac{\specnorm{\XXt-\XX}}{\sigma_{\min} (\XXt)}, \quad \normf{\VV_t\VV_t^\top-\VV\VV ^\top} \le \frac{\sqrt{2}\normf{\XXt-\XX}}{\sigma_{\min} (\XXt)};\\
    \specnorm{\WW_t\WW_t^\top-\WW\WW ^\top} &\le \frac{\specnorm{\XXt-\XX}}{\sigma_{\min} (\XXt)}, \quad \normf{\WW_t\WW_t^\top-\WW\WW ^\top} \le \frac{\sqrt{2}\normf{\XXt-\XX}}{\sigma_{\min} (\XXt)}.
\end{align*} 
\end{lemma}

The next lemma controls the perturbation of Moore–Penrose pseudoinverses of two matrices with the same rank.
\begin{lemma}\label{lemma:fake inverse}
    \cite[Theorem~4.1]{wedin1973perturbation}
Let $\AAf,\BB \in \mathbb{R}^{m \times n} $ satisfy $\rank(\AAf)=\rank(\BB)$. Then 
\begin{equation}
    \normf{\BB^+-\AAf^+} \le 3\specnorm{\BB^+}\specnorm{\AAf^+}\normf{\BB-\AAf}, 
    \label{eq:fake inverse inequality}
\end{equation}
where $\AAf^+$ denotes the Moore–Penrose pseudoinverse of  $\AAf$. In particular, if $\AAf=\UU\SSigma\WW^\top$ is the compact SVD of $\AAf$, then
\begin{equation*}
    \AAf^+=\WW \SSigma^{-1}\UU^\top
\end{equation*}

\end{lemma}

In addition, several useful properties of the RIP are summarized below. For completeness, a short proof is provided for the part not already covered in the literature.
\begin{lemma}\label{lemma: RIP}
     Let $ \mathcal{A}: \mathbb{R}^{n_1 \times n_2}\rightarrow \R^m$ be a linear measurement operator with RIP constant $\delta_r$. Then, 
     \renewcommand{\labelenumi}{\arabic{enumi}.}
     \begin{itemize}
         \item[(i)] Let $ \VV \in \R^{n_2 \times r'} $ and $ \UU \in \R^{n_1 \times r'} $ be any matrix with orthonormal columns, i.e., $ \VV^\top \VV = \II_{r'}$ and $ \UU^\top \UU = \II_{r'}$.  Then  any matrix $\ZZ \in \R^{n_1 \times n_2}$ with $\rank(\ZZ)\le r$, 
        \begin{align*}
            \normf{ \left(\IdOp -\Aops \right) (\ZZ) \VV} \le \delta_{r+r'} \normf{\ZZ},  \quad
            \normf{ \UU^\top \left(\IdOp -\Aops \right) (\ZZ)} \le \delta_{r+r'} \normf{\ZZ} .
        \end{align*}
        In particular, let  $r'=1$, it holds that
        \begin{align*}
            \specnorm{ \left(\IdOp -\Aops \right) (\ZZ) } \le \delta_{r+1} \normf{\ZZ}. 
        \end{align*}
        \item[(ii)]    Let $\ww \in  \R^{n_1}$ and $\vv \in  \R^{n_2}$ such that $\specnorm{\ww}=\specnorm{\vv}=1$,  and define the orthogonal projection operators
\begin{align*}
    \Projw (\ZZ) 
    &:= \nj{\ww \vv^\top, \ZZ} \ww \vv^\top,  \qquad 
    \Projwperp (\ZZ) := \ZZ - \nj{\ww \vv^\top, \ZZ} \ww \vv^\top.
\end{align*}
    Then, for any $\ZZ \in \R^{n_1 \times n_2}$ with $\rank(\ZZ)\le r$, it holds
        \begin{align*}
            \vert \nj{ \Aop (\ww \vv^\top), \Aop \left( \Projwperp (\ZZ) \right) } \vert 
            \le 
            \delta_{
            r+1} \normf{\ZZ}.
        \end{align*} 
     \end{itemize}
\end{lemma}
\begin{proof} 
All bounds except  $\normf{ \UU^\top \left(\IdOp -\Aops \right) (\ZZ)} \le \delta_{r+r'} \normf{\ZZ}$ are proven in  \cite[Lemma~2]{caifast}. It therefore suffices to establish this remaining inequality. For any $\ZZ_1, \ZZ_2 \in \R^{n_1 \times n_2 }$ with $\rank (\ZZ_1)=r$ and  $\rank (\ZZ_2)=r'$, it follows from \cite[Lemma~3.3]{candes2011tight} that 
     \begin{equation*}
        \vert \nj{ \left(\IdOp -\Aops\right) (\ZZ_1), \ZZ_2 } \vert
        \le 
        \delta_{r+r'} \normf{\ZZ_1} \normf{\ZZ_2}.
     \end{equation*}
Note that there exists a matrix $\MM \in \R^{r' \times n_2 }$ with $\normf{\MM}=1$
    such that 
\[
        \normf{\UU^{\top}\left( \IdOp - \Aops \right) (\ZZ)  }= 
        \nj{ \UU^{\top}\left[ \left( \IdOp - \Aops \right) (\ZZ) \right]  , \MM }
        =
        \nj{ \left[ \left( \IdOp - \Aops \right) (\ZZ) \right]  , \UU \MM }.
\]
Since $\rank(\UU \MM) \le r'$ and $\normf{\UU \MM} \le \specnorm{\UU} \normf{\MM} \le 1$, applying the above RIP inequality with $\ZZ_1=\ZZ, \ZZ_2=\UU \MM$ gives
    \begin{align*}
        \normf{\UU^{\top}\left( \IdOp - \Aops \right) (\ZZ) }
        \le 
        \delta_{r+r'} \normf{\ZZ} \specnorm{\UU} \normf{\MM} 
        =
        \delta_{r+r'} \normf{\ZZ}.
    \end{align*}
This proves the desired bound and completes the proof.
\end{proof}

\begin{lemma}\label{lemma:auxestimates}
Assume that the measurement operator $\mathcal{A}$ satisfies RIP with constant $\delta=\delta_{4r+1} \le 1$.
When $m\ge C(n_1+n_2)r$ for some universal constant $C>0$,  the following inequalities hold.
\renewcommand{\labelenumi}{\arabic{enumi}.}
\begin{itemize}
\item[(1)]
\begin{align*}
   &\left\| \left[ \left(  \Aops - \Aopws \right) \left(\XXstar- \XXt \right) \right]\WW_{t}^{\left(\ww,\vv \right)}\right\|_F \le
    2c^{\prime}\left(\specnorm{ \XXstar - \XXt }+\normf{\XXt-\XXtw}\right) \notag \\
    &\left\|{\VV_{t}^{\left(\ww,\vv \right)}}^{\top}\left[ \left(  \Aops - \Aopws \right) \left(\XXstar- \XXt \right) \right]\right\|_F \le
    2c^{\prime}\left(\specnorm{ \XXstar - \XXt }+\normf{\XXt-\XXtw}\right).
\end{align*}
\item[(2)] 
\begin{align*}
    \normf{
        \left[\left(   \Aopws - \IdOp \right) \left(  \XXtw - \XXt\right) \right] \WW_{t}^{\left(\ww,\vv \right)}
    } 
    &\le 2c^{\prime}\normf{  \XXtw - \XXt  } \notag \\
    \normf{
       {\VV_{t}^{\left(\ww,\vv \right)}}^{\top} \left[\left(   \Aopws - \IdOp \right) \left(  \XXtw - \XXt\right) \right] 
    } 
    &\le 2c^{\prime}\normf{  \XXtw - \XXt  },
\end{align*}
\end{itemize}
where $c^{\prime}:=\max \left\{ 
        \delta;
        8 \sqrt{2r\left(n_1+n_2\right)/{m}}
    \right\}$.
\end{lemma}
\begin{proof}

The lemma is a non-symmetric version of \cite[Lemma B.1]{stoger2025nonconvexmatrixsensingbreaking}. By definition of $\AAiw$, it holds $\nj{\AAiw, \Projw (\ZZ)} =0$ for all $\ZZ$. Consequently, one has
    \begin{align*}
    \left( \Aopws \right) \left( \Projw (\ZZ) \right)
   &=
    \frac{1}{m} \sum_{i=1}^{m}  \nj{\AAiw, \Projw (\ZZ)}\AAiw
    +\nj{\ww \vv^\top, \ZZ}\ww \vv^\top\\
    &=
    \nj{\ww \vv^\top, \ZZ}\ww \vv^\top,
    \end{align*}
and
\begin{align}
    &\left( \Aopws \right) \left( \Projwperp (\ZZ) \right) \notag \\
    &=
    \frac{1}{m} \sum_{i=1}^m  \nj{\AAiw, \Projwperp (\ZZ)} \AAiw 
    +
    \nj{\ww \vv^\top, \Projwperp (\ZZ) } \ww \vv^\top  \notag \\
    &\overeq{(a)}
    \frac{1}{m} \sum_{i=1}^m  \nj{\AAiw, \Projwperp (\ZZ)}\AAiw \overeq{(b)}\frac{1}{m} \sum_{i=1}^m  \nj{\AAi, \Projwperp (\ZZ)}\AAiw \notag \\
    &\overeq{(b)}
    \frac{1}{m} \sum_{i=1}^m  \nj{\AAi, \Projwperp (\ZZ)}\AAi
    -
    \frac{1}{m} \sum_{i=1}^m  \nj{\AAi, \Projwperp (\ZZ)} \nj{\ww \ww^\top, \AAi} \ww \vv^\top \notag \\
    &=
    (\Aops) \left( \Projwperp (\ZZ)\right) 
    - \nj{ \Aop (\ww \vv^\top), \Aop ( \Projwperp (\ZZ) ) } \ww \vv^\top, \label{eq:two crucial equalities}
\end{align}
where we used $\nj{\ww \vv^\top, \Projwperp (\ZZ) }=0$ in $(a)$,  and the definition of $\AAiw = \Projwperp (\AAi)$ in $(b)$. Recall that $\XXt=\LLt \RRt^\top$ and $\XXtw=\LLtw {\RRtw}^\top$. Then
\begin{align*}
    &(\Aopws) \left(\XXstar- \XXt \right)\\
    &=
    (\Aopws) \left( \Projw \left( \XXstar- \XXt \right) \right) 
    + (\Aopws) \left( \Projwperp \left(\XXstar- \XXt \right) \right)\\
    &=
    \Projw \left( \XXstar- \XXt \right)+(\Aops) \left(  \Projwperp \left( \XXstar- \XXt \right) \right)\\
    &-\nj{\Aop \left(\ww \vv^\top \right), \Aop \left( \Projwperp \left(\XXstar - \XXt\right) \right)} \ww \vv^\top,
\end{align*}
where the last equation follows from $\nj{\ww \vv^\top, \XXstar- \XXt}\ww \vv^\top=\Projw \left( \XXstar- \XXt \right)$ and equality \eqref{eq:two crucial equalities}. Hence
\begin{align*}
    \left(  \Aops - \Aopws \right) \left(\XXstar- \XXt \right)
    &=
    \left( \Aops - \IdOp \right)
    \left( \Projw \left( \XXstar- \XXt \right) \right) \nonumber \\
    &+\nj{\Aop (\ww \vv^\top), \Aop \left( \Projwperp (\XXstar - \XXt) \right)} \ww \vv^\top. 
\end{align*}

\paragraph*{Proof of (1).} Using the triangle inequality, we have
\begin{align}
    &\normf{\left(  \Aops - \Aopws \right) \left(\XXstar- \XXt \right) \WW_{t}^{\left(\ww,\vv \right)}} \nonumber \\
    &\le
    \normf{ \left( \Aops - \IdOp \right) \left( \Projw \left( \XXstar- \XXt \right) \right)\WW_{t}^{\left(\ww,\vv \right)}} 
    +
    \normf{\nj{\Aop (\ww \vv^\top), \Aop \left( \Projwperp (\XXstar - \XXt ) \right)} \ww \vv^\top} \nonumber \\
    &\overleq{(\textup{i})} 
    \delta \normf{ \Projw \left( \XXstar- \XXt \right) }
    + \big\vert \nj{ \Aop (\ww \vv^\top), \Aop \left( \Projwperp (\XXstar - \XXt)  \right) } \big\vert \nonumber\\
    &\overleq{(\textup{ii})} 
    \delta \specnorm{ \XXstar - \XXt }
    + \big\vert \nj{ \Aop (\ww \vv^\top), \Aop \left( \Projwperp (\XXstar - \XXtw)  \right) } \big\vert \nonumber \notag \\
    &+ \big\vert \nj{ \Aop (\ww \vv^\top), \Aop \left( \Projwperp (\XXtw-\XXt)  \right) } \big\vert \nonumber \\
    &\overleq{(\textup{iii})} 
    \delta \specnorm{ \XXstar - \XXt }
    + 4 \sqrt{\frac{n_1+n_2}{m}} \twonorm{\Aop \left( \Projwperp (\XXstar - \XXtw ) \right)}
    + \delta \normf{\XXt - \XXtw} \nonumber \\
    &\overleq{(\textup{iv})} 
    \delta \specnorm{ \XXstar - \XXt }
    + 4 \sqrt{\frac{2\left(n_1+n_2\right)}{m}} \normf{ \XXstar - \XXtw }
    + \delta \normf{\XXt-\XXtw} \nonumber \\
    &\overleq{(v)} 
    \left(\delta +   8 \sqrt{\frac{r\left(n_1+n_2\right)}{m}}  \right) \specnorm{ \XXstar - \XXt }
    +  \left( \delta  +   4 \sqrt{\frac{2\left(n_1+n_2\right)}{m}} \right) \normf{\XXt-\XXtw} \notag \\
    &\le 
    2 c^{\prime} \specnorm{ \XXstar - \XXt }+2 c^{\prime}\normf{\XXt-\XXtw} \notag .
\end{align}
Here, $(\textup{i})$ uses Lemma \ref{lemma: RIP} and  $\normf{\ww \vv^\top}\le 1$;  $(\textup{ii})$ uses the fact that 
 $\Projw$ is a rank-one projection matrix;  $(\textup{iii})$ follows from Lemma \ref{lemma:key} , Lemma \ref{lemma: RIP} and the fact that $\Projwperp$ is an orthogonal projection; $(\textup{iv})$ uses the RIP of rank $2r+1$;  $(\textup{v})$ uses $\rank(\XXstar - \XXt)\le 2r$ and  $\normf{ \XXstar - \XXtw }\le \normf{ \XXstar - \XXt }+\normf{ \XXt - \XXtw }$. The last inequality follows from the definition of $c^{\prime}:=\max \left\{ \delta; 8 \sqrt{2r\left(n_1+n_2\right)/m} \right\}$.
An identical argument with left multiplication by ${\VV_{t}^{\left(\ww,\vv \right)}}^{\top}$ yields
\begin{align*}
    &\normf{{\VV_{t}^{\left(\ww,\vv \right)}}^{\top}\left(  \Aops - \Aopws \right) \left(\XXstar- \XXt \right) } \le 2c^{\prime} \specnorm{ \XXstar - \XXt }+2c^{\prime}\normf{\XXt-\XXtw}.
\end{align*}

\paragraph*{Proof of (2).}
We similarly compute
\begin{align*}
   &\left(  \Aopws - \IdOp \right) \left( \XXtw - \XXt \right)\\
   &=
   \left(  \Aops - \IdOp \right) \left( \Projwperp \left( \XXtw - \XXt \right) \right)
   -
   \nj{ \Aop (\ww \vv^\top), \Aop \left( \Projwperp (\XXtw-\XXt) \right)  } \ww \vv^\top.
\end{align*}
Thus
\begin{align}
    &\normf{
        \left[\left(   \Aopws - \IdOp \right) \left( \XXtw-\XXt\right) \right] \WW_{t}^{\left(\ww,\vv \right)}
    } \nonumber \\
    &\overleq{(a)} 
    \delta \normf{ \Projwperp \left(\XXtw - \XXt \right) } 
    +
    \big\vert \nj{  \Aop (\ww \vv^\top), \Aop \left( \Projwperp (\XXtw - \XXt) \right) } \big\vert \nonumber \\
    &\overleq{(b)} 
    2\delta \normf{ \Projwperp (\XXtw - \XXt ) }\nonumber\\
    &\le 
    2c^{\prime}\normf{  \XXtw - \XXt  }, \nonumber
\end{align}
where  $(a)$ uses Lemma \ref{lemma: RIP} and $\specnorm{\ww \vv^\top}\le 1$, and $(b)$ follows from Lemma \ref{lemma: RIP}.  The same reasoning with ${\VV_{t}^{\left(\ww,\vv \right)}}^{\top}$ in place of $\WW_{t}^{\left(\ww,\vv \right)}$ yields the second inequality in (2). This completes the proof.
\end{proof}

\end{document}